\documentclass{article}

\usepackage[preprint]{neurips_2025}

\RequirePackage{algorithm}
\RequirePackage{algorithmic}
\RequirePackage{natbib}
\RequirePackage{forloop}

\usepackage{amsmath,amsthm,amsfonts}
\usepackage{subcaption}
\usepackage[algo2e,vlined,ruled,linesnumbered]{algorithm2e}
\usepackage{booktabs} 
\usepackage{xspace}

\usepackage{paralist}

\usepackage[utf8]{inputenc} 
\usepackage[T1]{fontenc}    
\usepackage{hyperref}       
\usepackage{url}            
\usepackage{booktabs}       
\usepackage{nicefrac}       
\usepackage{microtype}      
\usepackage{xcolor}         

\usepackage{cleveref} 

\newcommand{\E}{\mathbb E}
\newcommand{\abs}[1]{\left| #1 \right|}
\newcommand{\norm}[1]{\left\| #1 \right\|}
\newcommand{\pars}[1]{\left( #1 \right)}
\newcommand{\sqbr}[1]{\left[ #1 \right]}

\newcommand{\SqrNorm}[1]{\norm{#1}^2}

\newcommand{\Exp}[1]{\E \sqbr{#1}}
\newcommand{\ExpSqrNorm}[1]{\E \norm{#1}^2}

\newcommand{\ExpArg}[2]{\E_{#1} \sqbr{#2}}
\newcommand{\ExpArgSqrNorm}[2]{\E_{#1} \norm{#2}^2}
\newcommand{\ExpArgInnerProd}[2]{\E_{#1} \left\langle #2 \right\rangle}

\newcommand{\vc}{\mathbf c}

\newcommand{\vx}{\mathbf x}
\newcommand{\vy}{\mathbf y}

\newcommand{\x}[1]{\vx_{#1}}

\newcommand{\C}[2]{\vc^{(#1)}_{#2}}
\newcommand{\tC}[2]{\tilde{\vc}^{(#1)}_{#2}}

\newcommand{\f}[2]{f^{(#1)}_{#2}}
\newcommand{\G}[2]{G^{(#1)}_{#2}}
\newcommand{\step}{\eta}
\newcommand{\repr}[2]{r^{(#1)}_{#2}}
\newcommand{\dist}{\rho}

\DeclareMathOperator*{\avg}{avg}

\newtheorem{assumption}{Assumption}

\newtheorem{lemma}{Lemma}

\newtheorem{theorem}[lemma]{Theorem}

\newcommand{\algcomment}[1]{\texttt{\color{blue}// #1}}

\usepackage{array}
\newcolumntype{P}[1]{>{\centering\arraybackslash}p{#1}}
\newcolumntype{M}[1]{>{\centering\arraybackslash}m{#1}}

\usepackage{tikz}
\newcommand*\circled[1]{\tikz[baseline=(char.base)]{
    \node[shape=circle,draw,inner sep=0.3pt] (char) {#1};}}
    
\usepackage{caption}
\newif\ifcomm
\commfalse
\ifcomm
\newcommand{\minghao}[1]{\textbf{\color{red}Minghao: #1}}

\newcommand{\minlan}[1]{\textbf{\color{purple}Minlan: #1}}
\newcommand{\dima}[1]{\textbf{\color{blue}Dima: #1}}

\newcommand{\rana}[1]{\textbf{\color{teal}Rana: #1}}

\else
\newcommand{\minghao}[1]{}
\newcommand{\new}[1]{}
\newcommand{\minlan}[1]{}
\newcommand{\dima}[1]{}
\newcommand{\rana}[1]{}
\fi

\newcommand{\sysname}{\textsc{Fielding}\xspace}

\newif\ifshrink
\shrinkfalse
\ifshrink
\newcommand{\shrink}[1]{#1}
\else
\newcommand{\shrink}[1]{}
\fi

\usepackage[auth-lg]{authblk}

\title{\sysname: Clustered Federated Learning with Data Drift}

\author[1]{Minghao Li}
\author[2]{Dmitrii Avdiukhin}
\author[1]{Rana Shahout}
\author[3]{Nikita Ivkin}
\author[4,5]{Vladimir Braverman}
\author[1]{Minlan Yu}
\affil[1]{Harvard University}
\affil[2]{Northwestern University}
\affil[3]{Amazon}
\affil[4]{Johns Hopkins University}
\affil[5]{Google}

\begin{document}

\maketitle

\begin{abstract}

Federated Learning (FL) trains deep models across edge devices without centralizing raw data, preserving user privacy. However, client heterogeneity slows down convergence and limits global model accuracy. Clustered FL (CFL) mitigates this by grouping clients with similar representations and training a separate model for each cluster. In practice, client data evolves over time -- a phenomenon we refer to as \emph{data drift} -- which breaks cluster homogeneity and degrades performance. Data drift can take different forms depending on whether changes occur in the output values, the input features, or the relationship between them. We propose \sysname, a CFL framework for handling diverse types of data drift with low overhead.
\sysname detects drift at individual clients and performs selective re-clustering to balance cluster quality and model performance, while remaining robust to malicious clients and varying levels of heterogeneity. Experiments show that \sysname improves final model accuracy by 1.9–5.9\% and achieves target accuracy 1.16x–2.23x faster than existing state-of-the-art CFL methods.
\end{abstract}

\section{Introduction}
\label{sec:introduction}

Federated Learning (FL) is a distributed machine learning paradigm that enables multiple clients to collaboratively train a shared model while keeping raw data local~\citep{google_fl, fldecentralized}.
FL has gained attention due to its privacy-preserving nature, as it avoids transmitting sensitive data to a central server~\citep{flprivacy, papaya}.
This makes it particularly appealing in applications such as mobile text prediction~\citep{googlekeyboard}, energy forecasting~\citep{energyforecast}, and medical imaging~\citep{medicalimage}.

In practice, FL involves clients with diverse hardware, usage patterns, and data sources, leading to variation in both dataset sizes and distributions: known as data heterogeneity~\citep{dataheteroiniot, heteroswitch}. This heterogeneity slows model convergence and degrades model performance~\citep{li_ConvergenceFedAvg_2020, noniddwithsharing, robustflfromnoniid}. Clustered FL (CFL) addresses this by grouping clients with similar data characteristics and training a separate model per cluster~\citep{ifca, flexcfl, auxo, federateddistributeddrift}, improving learning efficiency and accuracy by reducing intra-cluster heterogeneity.

However, real-world deployments face data drifts over time.
Data drift typically falls into three categories~\citep{mlsysdesign}: (1) \emph{label shift}, where the output categories frequencies change but the characteristics associated with each category remain stable~\citep{labelshift_detecting, unified_labelshift}; (2) \emph{covariate shift}, where the distribution of the input features changes but the feature-label relationship remains constant~\citep{onlinefl_dynamicregret, matchmaker}; (3) \emph{concept shift}, where the underlying input-output relationship changes, rendering a model trained on past data inaccurate on new data~\citep{continualfl, federateddistributeddrift}.

As data drift accumulates, intra‐cluster heterogeneity increases, and clusters can become as diverse as the full client set~\citep{federateddistributeddrift, unifying_datashift}. Effective drift handling must address two key challenges: \textbf{adapting to varying drift magnitudes} while maintaining \textbf{practical efficiency}, and \textbf{supporting diverse drift types}.

\minlan{overall, check maybe 2.1 can be a bit simplified to avoid redundancy?}

\minlan{do you mean many drifts or both many and a few drifts/varying drift magnitudes? The arguments here should be aligned with 2.1.} 
\minghao{please see the new "first challenge" paragraph}

The first challenge arises when drifts have varying magnitudes. When many clients experience drift, re-clustering based on outdated average weights or gradients becomes invalid, often requiring global re-clustering. FlexCFL~\citep{flexcfl} and IFCA~\citep{ifca} attempt to adapt incrementally by moving clients one at a time while keeping the number of clusters fixed. They, however, struggle under large-scale shifts as they reassign drifted clients using clusters' current average gradient or weight, which shift significantly after massive reassignments.
Global re-clustering reinitializes all clusters to reflect current data, but incurs significant communication and computation costs: clients might download multiple model replicas and upload new gradients; moreover, the newly formed clusters typically suffer an accuracy dip. Hence, while global re-clustering works for all drift magnitudes, incremental reassignment is more desirable under small drifts for practical reasons. FedDrift~\citep{federateddistributeddrift} adopts global re-clustering and pays an amplified cost by having each client train on all cluster models. Auxo~\citep{auxo} and FedAC~\citep{fedac} reduce overhead by only reassigning and adjusting cluster counts with clients selected for training, but ignore drifted clients that are not selected, reducing overall effectiveness (see Section~\ref{subsec:combine_perclientmove_and_global}).

The second challenge is to detect and adapt to different types of drift. FL frameworks rely on client representations---client-side information used for clustering---to capture drifts. 
The label distribution vector is a common representation that reflects the label shift. Another is the input embedding: feature vectors derived from a client’s local data using a neural network, then aggregated (e.g., averaging) to represent the input distribution. When the input–output mapping remains stable, input feature shifts often correlate with label distribution shifts so that both representations can produce similar clustering results. However, neither captures changes in the input–output relationship itself -- concept drift. Addressing concept drift typically requires loss-based representations, such as loss values, gradients, or gradient directions. For example, \citet{CFL_modelagnostic} clusters clients based on gradient directions to identify such shifts.
Each representation has trade-offs: Label and embedding vectors capture only marginal distributions, while gradient-based features require additional computation, are sensitive to training noise, and evolve with model updates (Table~\ref{table:cluster_gradient} shows that gradient-based clustering generates better clusters as training progresses and the model becomes more stable). Selecting or combining these representations enables more reliable detection in all types of drift (see \Cref{app:representation}).
\minlan{This part on different drift using different metrics are too long. maybe move some to 2.1. instead, here, say a bit more on why it's hard to use different metrics, for clustering and for convergence}
\minghao{I prefer keeping the information that's already here; we already stated the problem with different metrics?}

In this work, we present \sysname, a CFL framework that handles multiple types of data drift with low overhead. \sysname makes two key design choices: it combines per-client adjustment with selective global re-clustering, adapting to varying drift magnitudes without incurring the full cost of global re-clustering at every step; and it re-clusters all drifted clients using lightweight, drift-aware representations, keeping client-side computation and communication costs low while enabling efficient information collection from all clients. Our theoretical framework supports all representation choices and provides per-round utility guarantees along with convergence bounds, allowing \sysname to effectively handle all three types of data drift.


To demonstrate practicality, we extended the FedScale engine~\citep{fedscale} to support streaming data and built a \sysname prototype on top. We evaluated our framework on four image streaming traces with up to 5,078 clients experiencing various drift patterns. In these scenarios, \sysname improves the final accuracy by up to 5.9\% and reaches target accuracy as much as 2.23× faster than standard CFL approaches. Moreover, we show that \sysname flexibly accommodates different drift types through pluggable client representations 
(e.g., label distribution vectors, input embeddings, and gradients), integrates seamlessly with a range of client selection and aggregation schemes, and remains robust against malicious clients and varying degrees of heterogeneity.


\section{\sysname}
\label{sec:design_decisions}
\sysname is a CFL framework that handles multiple types of data drift with low overhead. Managing data drifts effectively requires identifying clients that need reassignment, determining their appropriate clusters, and collecting information to support drift detection and client movement. We begin by motivating the challenges of drift handling, followed by the design decisions that shape \sysname, and then provide an overview of the system.

\subsection{Motivation and design choices}
\label{subsec:combine_perclientmove_and_global}

We analyze the limitations of existing clustering strategies under real-world drift to motivate the design of \sysname. Our analysis uses the Functional Map of the World (FMoW) dataset~\citep{fmow}, containing 302 clients (one per unique UTM zone) with time-stamped satellite images labeled by land use (e.g., airport, crop field). Two training rounds correspond to one day. More details on FMoW appear in Section~\ref{sec:evaluation}. To quantify intra-cluster heterogeneity, we use \emph{mean client distance}~\citep{oort}: for each client, we compute the average pairwise L1 distance between its data distribution and that of clients in the same cluster. We then take the mean across clients.
\dima{Handle other representations?}

\begin{figure}
\begin{minipage}[b]{0.45\textwidth}
\begin{figure}[H]
    \centering 
    \includegraphics[width=\linewidth]{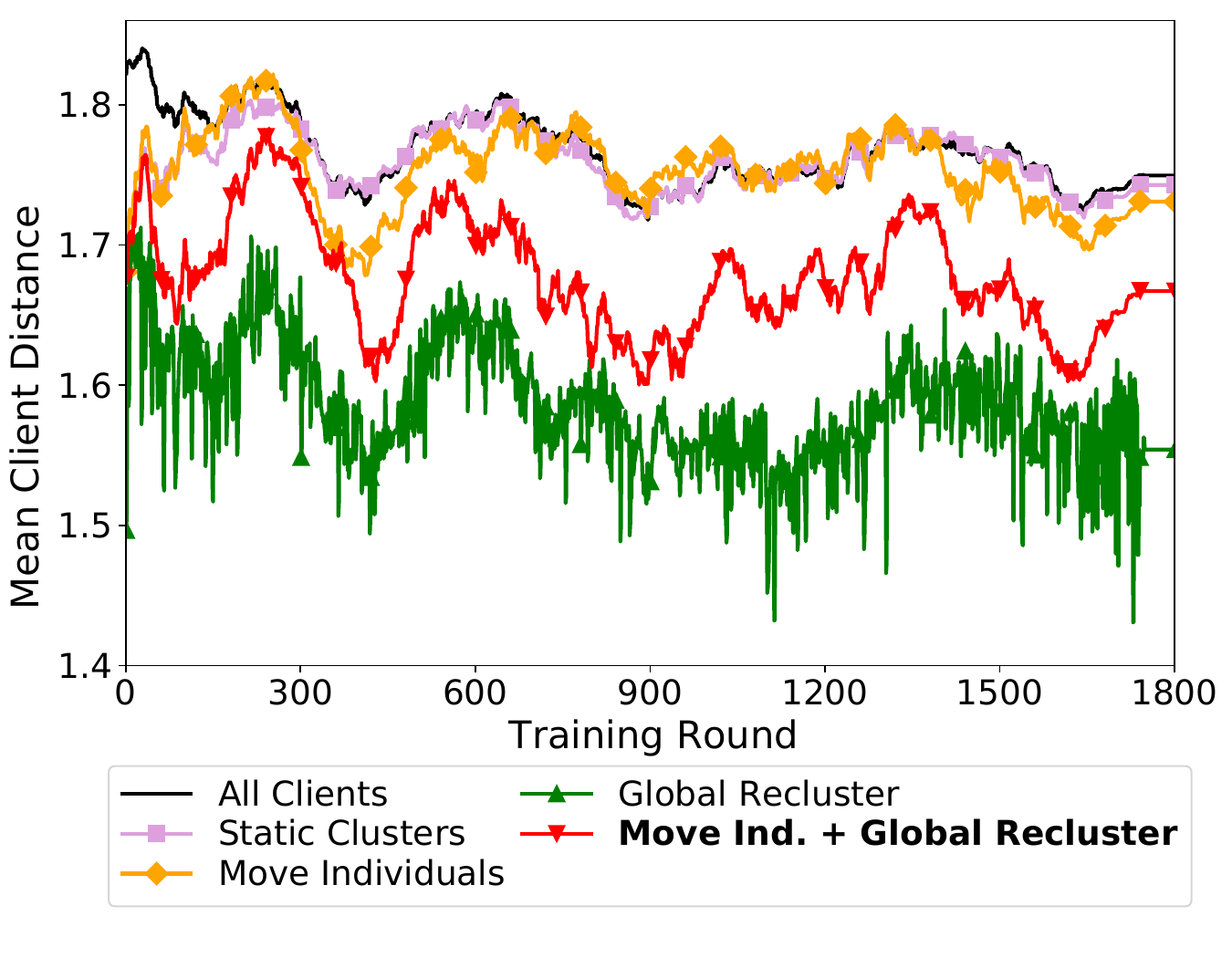}
    \caption{Client heterogeneity of the global set and label distribution-based clusters.
    }
    \label{fig:fmow_intra_cluster_hetero}
\end{figure}
\end{minipage}
\hfill
\begin{minipage}[b]{0.54\textwidth}
\begin{figure}[H]
    \centering
    \begin{subfigure}[b]{0.49\textwidth}
        \centering
        \includegraphics[width=\linewidth]{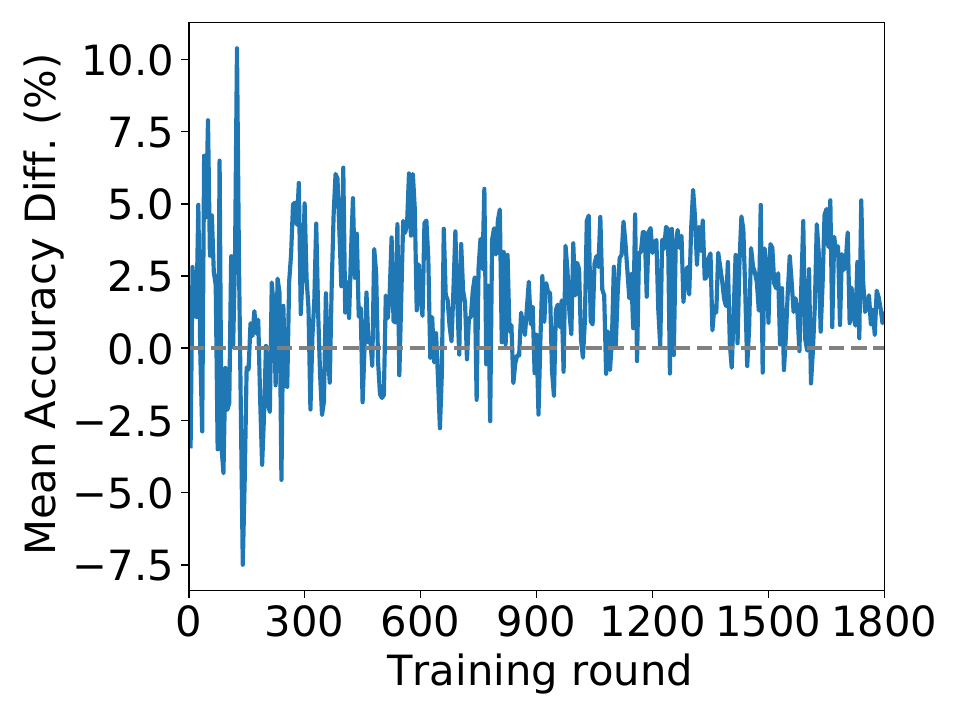}
        \caption{Selective vs. default global re-clustering (positive = selective better).
        }
        \label{fig:global_vs_selective}
    \end{subfigure}
    \begin{subfigure}[b]{0.49\textwidth}
        \centering
        \includegraphics[width=\linewidth]{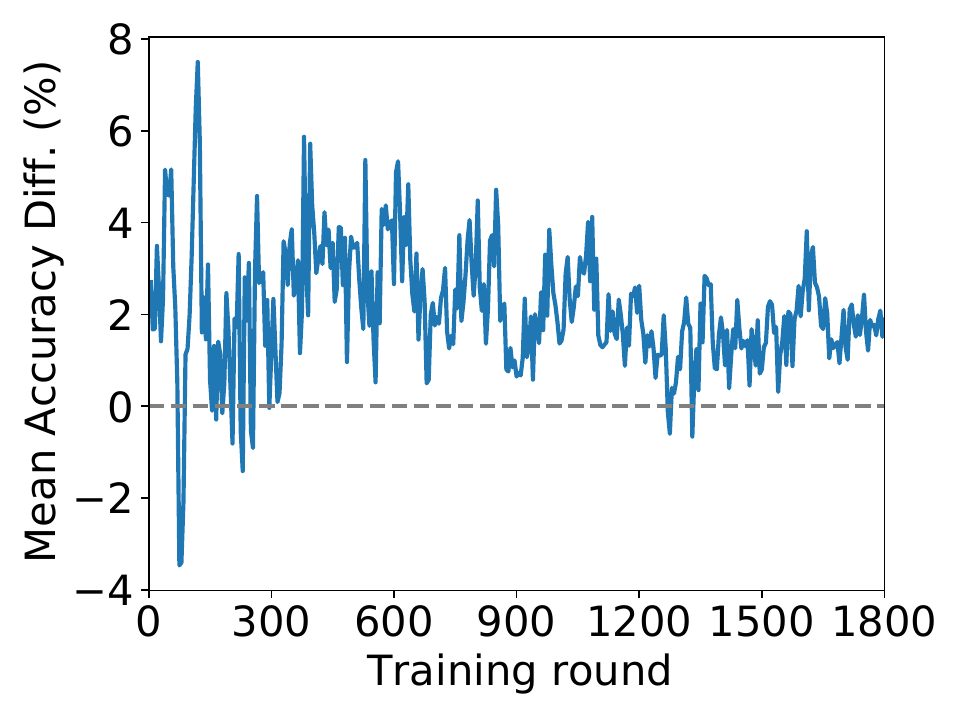}
        \caption{Global re-clustering vs. Selected clients re-clustering (positive = global better). 
        }
        \label{fig:all_vs_selected}
    \end{subfigure}
    
    \caption{Accuracy difference between different re-clustering approaches.}
    \label{fig:accuracy_difference}
\end{figure}
\end{minipage}
\end{figure}

\textit{Per-client adjustment fails under many drift.}
Per-client adjustment methods, such as IFCA~\citep{ifca} and FlexCFL~\citep{flexcfl}, maintain a fixed number of clusters and reassign drifted clients individually based on gradients or model weights. As mentioned in \Cref{sec:introduction}, such assignments become unstable and lead to poor clustering when many clients drift simultaneously.\minghao{removed: This approach becomes ineffective when many clients drift simultaneously. In these cases, the basis for reassignment, the average weights or gradients of each cluster, shifts rapidly, making assignments unstable and leading to poor clustering.} Figure~\ref{fig:fmow_intra_cluster_hetero} (“Move Individuals”) highlights this behavior. Between rounds 600–730, a significant drift event increases intra-cluster heterogeneity, and from rounds 740–1340, per-client adjustment fails to restore effective clustering. In fact, heterogeneity exceeds that of the unclustered client set.

\textit{Global re-clustering is unstable during small drifts.}
Global re-clustering resets all client assignments to reflect the current data, and it produces the lowest heterogeneity in our experiments (Figure~\ref{fig:fmow_intra_cluster_hetero}, “Global Recluster”). However, due to the random initialization of k-means~\citep{kmeansplusplus, kmeans_overview}, the resulting clusters can vary considerably across rounds, even under small drifts. Combined with cluster warm-up, this leads to test accuracy fluctuations. Figure~\ref{fig:global_vs_selective} shows the accuracy difference between our selective re-clustering approach (described in Section~\ref{sec:overview}) and a baseline that triggers global re-clustering after every drift event. The baseline exhibits unstable accuracy, sometimes matching ours but falling short by up to 5\%.

\textbf{Design choice: combine per-client adjustment with selective global re-clustering.}
Per-client adjustment breaks in highly dynamic environments where many clients experience drift. Meanwhile, global re-clustering can harm model accuracy by destabilizing cluster assignments and slowing convergence, particularly during minor drifts.
To balance these trade-offs, \sysname combines both strategies: starting with individual client migration and triggering global re-clustering only when needed (see Section~\ref{sec:overview} for details). As shown by the red curve (“Move Ind.\,+\,Global Recluster”) in Figure~\ref{fig:fmow_intra_cluster_hetero}, this approach maintains low intra-cluster heterogeneity throughout training (see Section~\ref{sec:evaluation}).

\textit{Clustering only selected clients reduces accuracy.}
Methods such as Auxo~\citep{auxo} and FedAC~\citep{fedac} use training outputs—e.g., gradients or model parameters—from selected clients to re-cluster them without any additional cost. However, unselected clients do not provide such information, making it unclear whether they have drifted or where they should be reassigned to. As outlined in Appendix~\ref{subsec:CFL}, typical FL settings sample only a subset of clients in each round. When unselected clients drift but remain in their previous clusters, they may receive models trained on data with different distributions, reducing accuracy. Figure~\ref{fig:all_vs_selected} compares the mean test accuracy when re-clustering only selected clients versus all drifted clients in each round. Re-clustering all drifted clients yields a 1.5\%–4\% accuracy gain across most rounds—a meaningful improvement given the final FMoW accuracy of 52.4\%. These results show that re-clustering only selected clients produces imperfect clusters containing misaligned clients.

\textit{Global re-clustering based on training outputs incurs high overheads.}  
Handling all drifted clients requires input from every client, not just those selected for training. Depending on the representation, this method might incur significant overhead. 
\minghao{removed: One approach is to train a global model on selected clients and send it to all clients for local gradient computation.} In FlexCFL~\citep{flexcfl}, clients train the global model locally and report gradients to the coordinator. In our experiments using ResNet-18~\citep{resnet} on FMoW, clients spent on average 116.7 seconds for model download and gradient upload, and an additional 50.4 seconds running forward and backward passes on local data.

\textbf{Design choice: efficient re-clustering of all drifted clients across drift types.} To handle drift comprehensively and efficiently, \sysname re-clusters all drifted clients, regardless of whether they were selected for training. Instead of relying on training outputs, it uses lightweight representations that can be collected from all clients with minimal overhead. To support different drift types while maintaining low cost, \sysname allows configurable client representations: it uses label-distribution vectors—small and easy to compute—for detecting label and covariate shifts, and relies on gradients only when needed to capture concept drift. This design avoids accuracy loss caused by stale cluster assignments and eliminates the cost of global model propagation and local retraining.

\subsection{System overview}
\label{sec:overview}
\sysname consists of two core components: a centralized coordinator that manages client representations, performs clustering, and orchestrates model training; and a client-side module that tracks local representations and reports updates when drift is detected. As discussed in Section~\ref{subsec:combine_perclientmove_and_global}, neither per-client adjustment nor global re-clustering alone performs well across all drift scenarios. \sysname adopts a hybrid approach: it uses per-client adjustment by default and triggers global re-clustering only when necessary based on a threshold~$\tau$. When clients detect drift, they send updated representations to the coordinator. The coordinator initially reassigns each drifted client to the nearest cluster without updating cluster centers during this phase to ensure deterministic outcomes regardless of client processing order. After all individual adjustments, the coordinator recalculates cluster centers and measures the shift. If any center has moved more than~$\tau$, global re-clustering is triggered for all clients; otherwise, the updated cluster membership is finalized. In our prototype, setting~$\tau$ to one-third of the average inter-cluster distance worked empirically well. Figure~\ref{fig:clustering_workflow} illustrates this clustering process using label-distribution vectors as representation. Clients submit their representations upon registration and whenever drift is detected. The coordinator updates client metadata and assigns each client to the nearest cluster. Before selecting participants for the next round, it checks whether any cluster centers have shifted beyond~$\tau$ to determine if global re-clustering is needed.

\begin{figure}
    \centering 
    \includegraphics[width=0.8\linewidth, trim={0.2cm 5.4cm 4.1cm 0.1cm},clip]{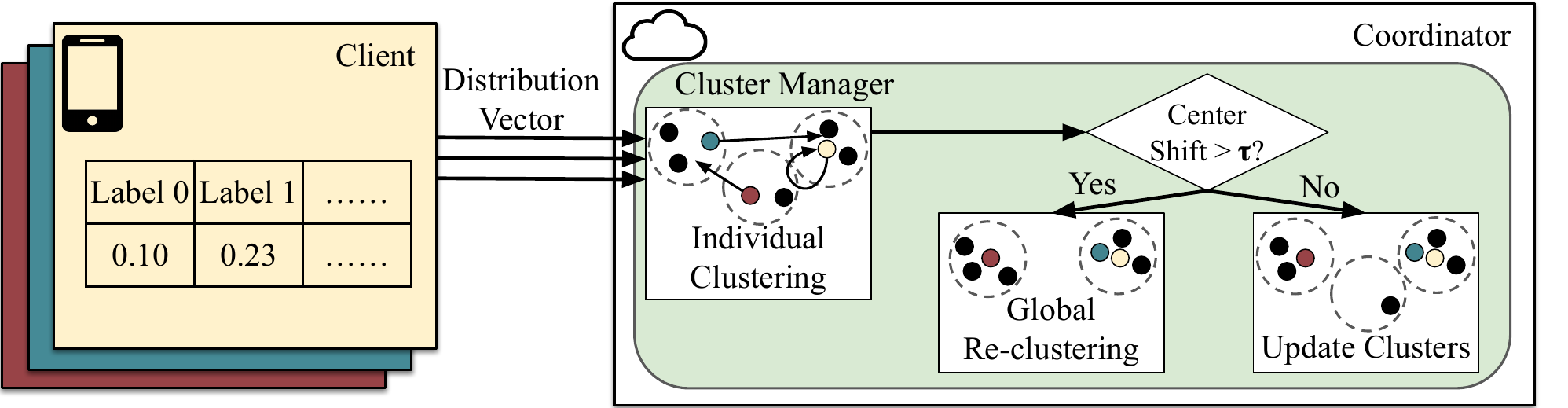}
    \caption{Clustering overview of \sysname with label distribution as client representation. Clients send distribution vectors to the coordinator; the coordinator moves drifted clients to the closest cluster and triggers global re-clustering if any cluster center shifts by a distance larger than $\tau$. 
    }
    \label{fig:clustering_workflow}

\end{figure}





\subsection{Algorithmic framework}
\label{algorithm}

In this section, we describe our algorithm~-- presented in \Cref{alg:main}~-- in detail.
Initially, \sysname clusters the clients using the $k$-means algorithm.
The distance between clients is measured based on clients representations; in particular, for label-based clustering, we measure distance between data distributions, which is computed as the $\ell_1$-distance between their histograms: that is, assuming we have $l$ labels, the distance between clients with label distribution $p_1, \ldots, p_l$ and $q_1, \ldots, q_l$ is defined as $\sum_{i=1}^l |p_i - q_i|$.

We assume that the algorithm is executed for $T$ global iterations, and during each iteration, the label distribution of each client might change, for instance by adding or removing data points.
The first step of \sysname is to handle data drift (\Cref{alg:data_drift}).
Our goal is, by the end of each iteration, to maintain a good model for each cluster, with respect to its clients' data.
In this step, each drifted client is assigned to the closest cluster.
If this causes a significant change in cluster centers~-- namely, if there exists a cluster moving by more than $\theta / 3$, where $\theta$ is the average distance between cluster centers~-- we recluster clients from scratch.
This condition avoids frequent reclustering, which might require excessive resources and can adversely change the clients' losses.
Importantly, we recluster all the clients, not just available clients, which, as we show in \Cref{sec:motivation}, improves the accuracy.

For the global reclustering, we choose the number of clusters that gives the highest silhouette score.
After reclustering, we compute new cluster models: for each client $i$, let $\vx_i$ be its old cluster model, and then for each new cluster $C_k$, we define its model as an average of $\vx_i$ for $i \in C_k$ (i.e., we average the old client models).\shrink{\par} After the clusters are computed, we train the cluster models for $R$ rounds.
In each round, we sample clients from the cluster and run $L$ local iterations of gradient descent.
After local iterations, we set the cluster model as the average of the models of all participating clients.

\begin{figure}[t]
\begin{minipage}[t!]{0.48\textwidth}
\begin{algorithm}[H]
    \caption{CFL with data drift}
    \label{alg:main}
    \begin{flushleft}
    Partition clients into clusters using $k$-means clusters based on client representations \\
    Initialize cluster models $\vc^{(1)}, \ldots, \vc^{(K)}$ \\
    \For{every iteration $t = 0, \ldots, T-1$} {
        Handle data drift using \Cref{alg:data_drift} \\
        Let $C_1, \ldots, C_K$ be the clusters \\
        Let $M$ be the number of machines sampled per iteration \\
        \For{every cluster $k = 1, \ldots, K$ in parallel} {
            \For{$R$ rounds} {
                Sample a set $S$ of $M / K$ clients from $C_k$ based on the selection strategy \\
                \For{each client $i \in S$ in parallel} {
                    Initialize $\vx_i = \vc^{(k)}$ \\
                    \For{$L$ local iterations} {
                        $\vx_i \gets \vx_i - \step \G{i}{t}(\vx_i)$
                    }
                }
                $\vc^{(k)} \gets \avg_{i \in S} \vx_i$
            }
        }
        \textbf{report} cluster models and client-to-cluster assignment  
    }
    \end{flushleft}
\end{algorithm}
\end{minipage}
\hfill
\begin{minipage}[t!]{0.48\textwidth}
\begin{algorithm}[H]
    \caption{Handling data drift}
    \label{alg:data_drift}
    \begin{flushleft}
    \For{each drifting client $x$} {
        Assign $x$ to the closest cluster center \\
        Recompute the centers of the affected clusters
    }
    Let $\theta$ be the average distance between the cluster centers \\
    \If{some center moved by at least $\theta / 3$} {
        \For{each client $i$}{
            Let $\vx_i$ be the model corresponding to the $i$'th client's cluster
        }
        Let $C_1, \ldots, C_K$ be $k$-means clusters based on client representations \\
        \For{each cluster $C_k$} {
            $\vc^{{(k)}} \gets \avg_{i \in C_k} \vx_i$
        }
    }
    \end{flushleft}
\end{algorithm}
\begin{algorithm}[H]
    \caption{Clustering}
    \label{alg:clustering}
    \begin{flushleft}
    Define the distance between clients based on their representations \\
    Choose $K$ with the largest silhouette score \\
    Cluster the clients using the $K$-means clustering
    \end{flushleft}
\end{algorithm}
\end{minipage}
\end{figure}

\subsection{Convergence}
\label{algosec:convergence}

In \Cref{app:convergence}, we analyze the performance of simplified version of our framework by bounding the average of clients' loss functions at every iteration.
The major challenge in the analysis is bounding the adverse effects of data drifts and reclustering.
Data drifts change client datasets, hence changing their local objective functions.
On the other hand, while reclustering is useful in the long run, its immediate effect on the objective can potentially be negative (see \Cref{fig:all_vs_selected}) due to the mixing of models from different clusters.

We make the following assumptions.\footnote{See \Cref{app:convergence} for the precise statements and the discussion of the assumptions}
First, we assume that the distance between client representations~-- for example label distributions~-- translates into the difference between their objective functions.
Second, we assume that the effect of data drift is bounded: that is, the representation $r_i$ of client $i$ changes by at most $\delta$ for each data drift.
Finally, we assume that clients are clusterable; that is, there exists $K$ clusters so that representations of clients within each cluster are similar.
\begin{assumption}
    \label{ass:all}
    Each $\f{i}{t}$~-- the local objective for client $i$ at iteration $t$~-- is $L$-smooth and satisfies $\mu$-Polyak-\L{}ojasiewicz condition, and we have access to a stochastic oracle with variance $\sigma^2$.
    There exists clustering such that representations inside each cluster are $\Delta$-close, the data drift changes representations by at most $\delta$, and the ratio between the difference between the objective function values and the distance between representations is at most $\theta$.
\end{assumption}
Under these assumptions, data drifts can affect each client representation~-- and hence the objective we optimize~-- by at most a fixed value.
Moreover, by the clustering assumption, clients within each cluster have similar objectives after the reclustering, allowing us to bound the increase in the loss function due to reclustering.
In Appendix~\ref{app:convergence}, we show the following result.
\begin{theorem}
    Let $N$ be the number of clients and $M$ be the total number of machines sampled per round.
    Let $\vx^*$ be the minimizer of $f_0 = \avg_i \f{i}{0}$.
    Let $\C{k,*}{t}$ be the minimizer for cluster $k$ at iteration $t$.
    Then, under \Cref{ass:all}, for $\step \le 1/L$, for any iteration $t$ we have
    \begin{align*}
        \frac{1}{N} \sum_{k \in [K]} \sum_{i \in C_k} \pars{\f{i}{t}(\C{k}{t + 1}) - \f{i}{t}(\C{k,*}{t})}
        &\le (1 - \step \mu)^{tR} \pars{f_0(\x{0}) - f_0(\vx^*)} \\
        &\quad + \frac{L \step}{2\mu} \pars{\frac{\sigma^2 + 8 L \theta \Delta}{M/K} + 3 \theta (\Delta + \delta) (1 - \step \mu)^{R}}
    \end{align*} 
\end{theorem}
Intuitively, at every iteration, we provide a ``regret'' bound, comparing the loss at each cluster with the best loss we could have at each cluster.
The bound has two terms: an exponentially decaying term corresponding to the initial loss, and a non-vanishing term corresponding to stochastic noise and the loss due to clustering and data drift.
Ultimately, at early stages, the first term dominates, and we observe rapid improvement in objective value.
On the other hand, at later stages, the first term vanishes, and the behavior is dictated by the stochastic and sampling noise, and severity of data drifts.

\shrink{Compared with many previous works with provable guarantees~\citep{ifca,personal_fl_approaches}, we don't require all participating clients to compute the loss function on all machines.
The closest work is~\citet{ma_ConvergenceClustered_2022}, which doesn't handle data drifts and assumes that the gradients inside each cluster are close~-- an assumption about the algorithm behavior (considering that the clustering is performed based on model parameters), not about the problem parameters.
\dima{Not sure I like how this sounds. I say this to emphasize that it's not the case that that paper made some assumptions, and we made some set of incomparable assumptions.}}

\section{Evaluation}
\label{sec:evaluation}

We evaluate \sysname prototype on four FL tasks with label distribution vectors as client representations (see \Cref{app:representation} for results with other representations). End-to-end results highlight that \sysname improves final test accuracy by 1.9\%-5.9\% and is more stable than prior CFL methods. It works well with complementary FL optimizations, remains robust against malicious clients reporting corrupted distribution vectors, and still accelerates model convergence under low heterogeneity. 

\textbf{Environment.} We emulate large-scale FL training with two GPU servers, each with two NVIDIA A100 GPUs (80 GB memory) and two AMD EPYC 7313 16-core CPUs. We use FedScale's device datasets for realistic device computing and network capacity profiles while following the standardized training setup in the original paper~\citep{fedscale}. Per-device computation time is estimated from device computing speed, batch size, and number of local iterations; communication time is estimated based on the volume of downloaded and uploaded data and network bandwidth.

\textbf{Datasets and models.}
We use a satellite image dataset (FMoW~\citep{fmow}) and two video datasets (Waymo Open~\citep{waymo} and Cityscapes~\citep{cityscapes}) that exhibit natural data drift. For FMoW, we keep images taken after January 1, 2015 and create one client per UTM zone. Two training rounds correspond to one day. For Waymo Open and Cityscapes, we use video segments pre-processed by Ekya~\citep{ekya}. We create one client per camera per Waymo Open segment (212 clients in total), and one client per 100 consecutive frames per Cityscapes segment (217 clients in total).  As these two datasets have video frame IDs within segments but lack global timestamps, we sort samples by frame ID and split them into 10 intervals. We stream in one data interval every 30 rounds with Cityscapes and every 20 rounds with Waymo Open. We train ResNet-18~\citep{resnet} on Cityscapes and FMoW, and VisionTransformer-B16~\citep{visiontransformer} on Waymo Open.

To evaluate \sysname under highly dynamic drift, we construct a synthetic trace using the FedScale Open Images~\citep{openimages} benchmark. We find the top 100 most frequent classes and retain clients with samples from at least 10 of these 100 classes (5078 clients in total). Each client randomly partitions local data \textit{labels} into 10 buckets, with each bucket then containing all samples of its labels. We stream in one bucket every 50 rounds and use ShuffleNet v2~\citep{shufflenetv2} for this task. For all four datasets, clients start with 100 rounds worth of data and retain samples from the most recent 100 rounds. Dataset and training configurations are summarized in Table~\ref{table:dataset_setup} and Table~\ref{table:training_parameter}.

\textbf{Baselines.} Our non-clustering baseline trains one global model by randomly selecting a subset of participants from all available clients every round. To compare with other clustered FL works with data drift handling measures, we use Auxo~\citep{auxo} as the continuous re-clustering baseline and FlexCFL~\citep{flexcfl} as the individual movement baseline. We employ FedProx~\citep{fedprox}, a federated optimization algorithm that tackles client heterogeneity, for all approaches. 

\textbf{Metrics.} Our main evaluation metrics are Time-to-Accuracy (TTA) and final test accuracy. We define TTA as the training time required to achieve the target accuracy--the average client test accuracy when training a single global model (black dashed lines in Figure~\ref{fig:tta_all}). We report the final average test accuracy across diverse settings to demonstrate the sensitivity and robustness of \sysname.

\begin{figure*}[t!]
    \centering
    \includegraphics[clip, trim=2.9cm 2.9cm 1cm 0cm, width=\textwidth]{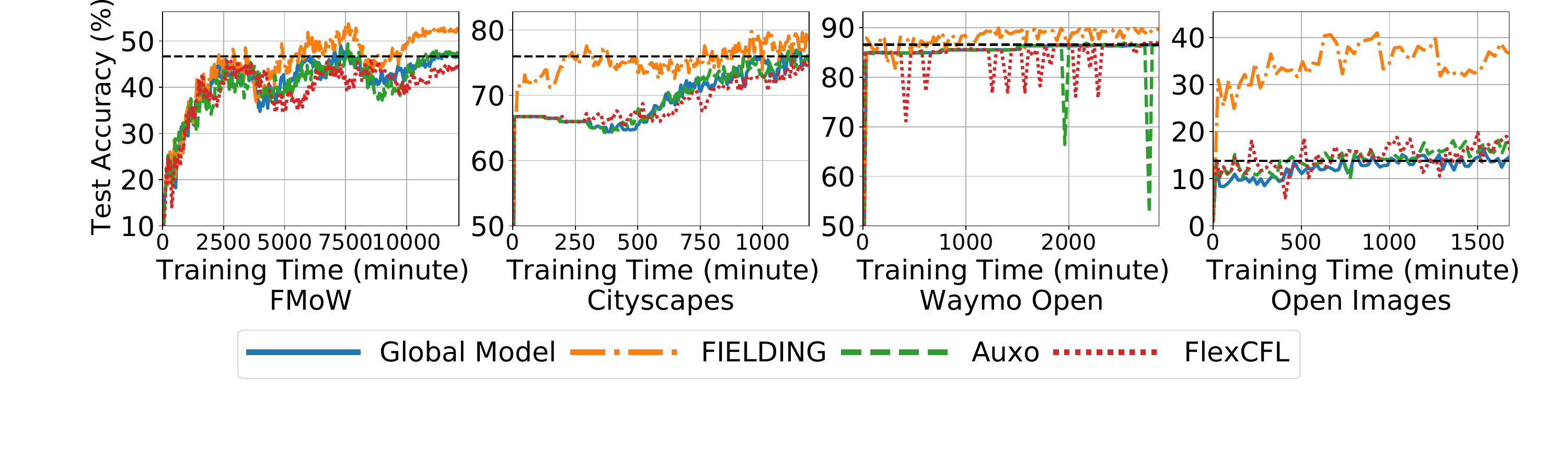}
    \caption{Time to accuracy (TTA) comparison over four tasks.}
    \label{fig:tta_all}
\end{figure*}

\subsection{End-to-end training performance}
\label{sec:eval:e2e}
Figure~\ref{fig:tta_all} shows that \sysname improves average client test accuracy by 5.9\%, 1.9\%, and 3.4\% in FMoW, Cityscapes, and Waymo Open, respectively. These improvements are significant, given that baseline test accuracies range from 45\% to 85\%, and are comparable to the accuracy gains prior CFL works achieve on \textit{static data}. When considering the global model's final test accuracy as the target accuracy, \sysname offers a 1.27$\times$, 1.16$\times$, and 2.23$\times$ training speed-up, respectively (We define the point at which \sysname's accuracy consistently surpasses the target accuracy as the moment it achieves the target accuracy). On the highly heterogeneous and dynamic Open Images trace, \sysname boosts accuracy by 17.9-26.1\% in the final 100 rounds.

While Auxo performs comparably or slightly better than the global model baseline, the benefits are limited: up to 0.9\% accuracy gain and up to 1.05$\times$ convergence speedup. These results fall short of those reported in the original paper on static datasets, suggesting a decline in clustering effectiveness. This degradation stems from re-clustering only the selected participants after each round, ignoring drifted but unselected clients. This approach also leads to sudden accuracy drops on the biased Waymo Open dataset, where over 80\% samples are cars, due to sudden model shifts when divergent clients are selected. FlexCFL shows no improvement in accuracy or convergence across all real-world datasets, as it migrates clients individually without adjusting the number of clusters. As Figure~\ref{fig:fmow_intra_cluster_hetero} shows, this approach leads to cluster heterogeneity approaching that of the global client set.

\textbf{FedDrift Comparison.}
\label{sec:eval:feddrift}
FedDrift~\citep{federateddistributeddrift} is another clustered FL method designed to handle drift. However, it is impractical for large-scale settings, as it (1) requires every client to train every cluster model in each round and (2) assumes all clients remain online and available for training. Therefore, we construct a small-scale setting with the FMoW dataset and 50 clients (the number of clients we selected per round in our original experiment). Figure~\ref{fig:feddrift_comparison} shows that \sysname achieves 2.5\% higher final accuracy and reaches FedDrift's final accuracy 2.01$\times$ faster.


\subsection{Compatibility with other optimizations}

To demonstrate that \sysname is agnostic to client selection and model aggregation strategies, we evaluate \sysname's performance on FMoW together with various complementary optimizations. In Figure~\ref{fig:selection} and Figure~\ref{fig:aggregation}, we use the algorithm name alone to denote the baseline where we train one global model and use "+ \sysname" to denote running the algorithm atop \sysname.

We use the Oort selection algorithm~\citep{oort} and a distance-based algorithm prioritizing clients closer to the distribution center as client selection examples. As shown in Figure~\ref{fig:selection}, \sysname improves the final average test accuracy by 6.0\% and 4.5\% and reach the target accuracy $2.62\times$ and $1.17\times$ faster. Note that Oort selection actively incorporates clients otherwise filtered out due to long response time, leading to a longer average round time. Figure~\ref{fig:aggregation} shows that \sysname works well with aggregation algorithm FedYogi~\citep{fedyogi} and q-FedAvg~\citep{qfedavg}. \sysname improves the final accuracy by 5.9\% and 4.9\% while giving a $1.34\times$ and $1.25\times$ speedup respectively.


\subsection{Robustness and sensitivity analysis}
\textbf{Malicious clients.} We study the impact of having malicious clients who intentionally report wrong label distribution vectors. We simulate such behavior by randomly selecting a subset of clients who permute the coordinates of their distribution vectors when registering with the coordinator. As reported in Figure~\ref{fig:robustness}, \sysname consistently outperforms the baseline global model accuracy across different percentages of malicious clients.

\begin{figure}[t!]
    \begin{minipage}[t]{0.49\textwidth}
    \begin{figure}[H]
        \centering
        \begin{subfigure}[t]{0.49\textwidth}
            \centering
            \includegraphics[width=\linewidth, trim={0.4cm 2.3cm 0.38cm 0.1cm}, clip]{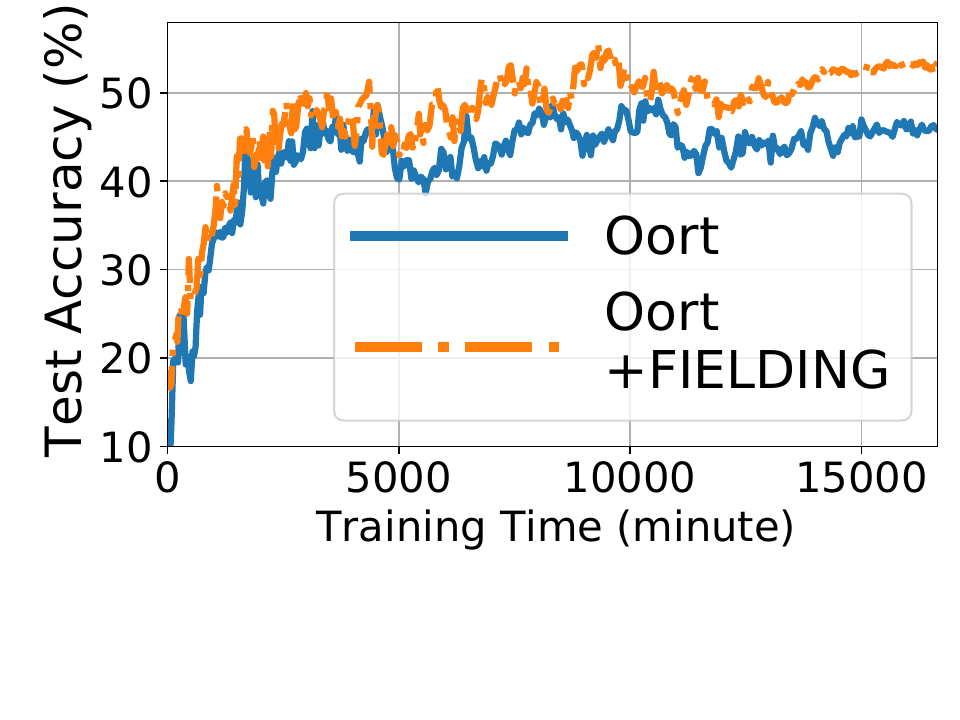}
            \label{fig:selection_oort}
            \caption{Oort selection}
        \end{subfigure}
        \hfill
        \begin{subfigure}[t]{0.49\textwidth}
            \centering
            \includegraphics[width=\linewidth, trim={0.4cm 2.3cm 0.38cm 0.1cm}, clip]{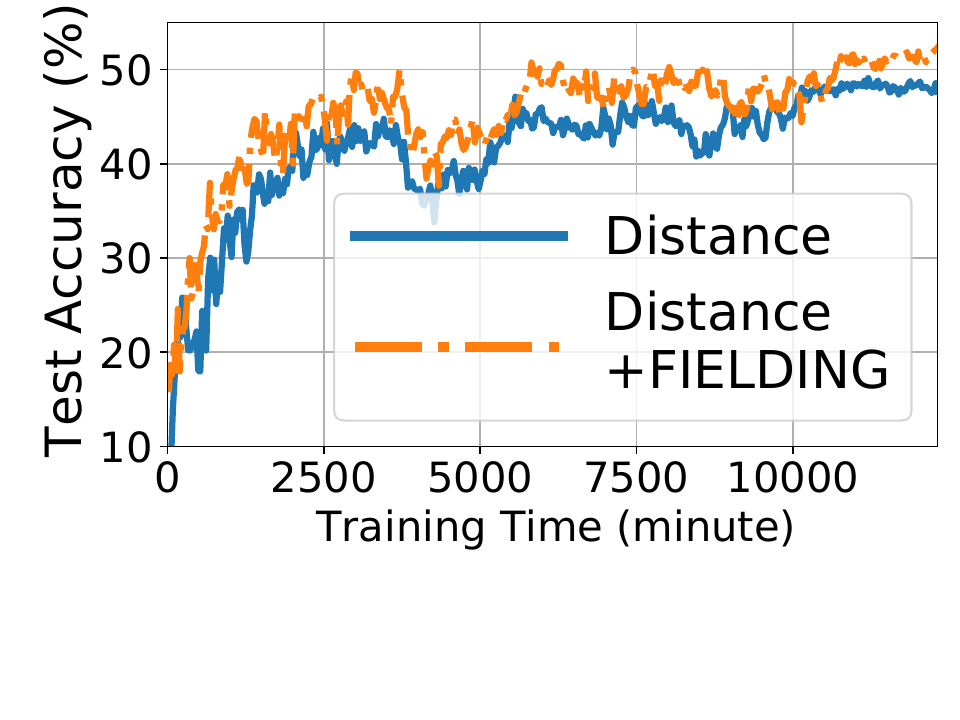}
            \begin{minipage}[t]{1.4in}
              \label{fig:selection_distance}
            \caption{Distance-based selection}
          \end{minipage} 
        \end{subfigure}
         \caption{\sysname with different client selection strategies on FMoW dataset.}
         \label{fig:selection}
    \end{figure}
    \end{minipage}
    \hfill
    \begin{minipage}[t]{0.5\textwidth}
    \begin{figure}[H]
        \centering
        \begin{subfigure}[t]{0.49\textwidth}
            \includegraphics[width=\linewidth, trim={0.4cm 2.3cm 0.38cm 0.1cm}, clip]{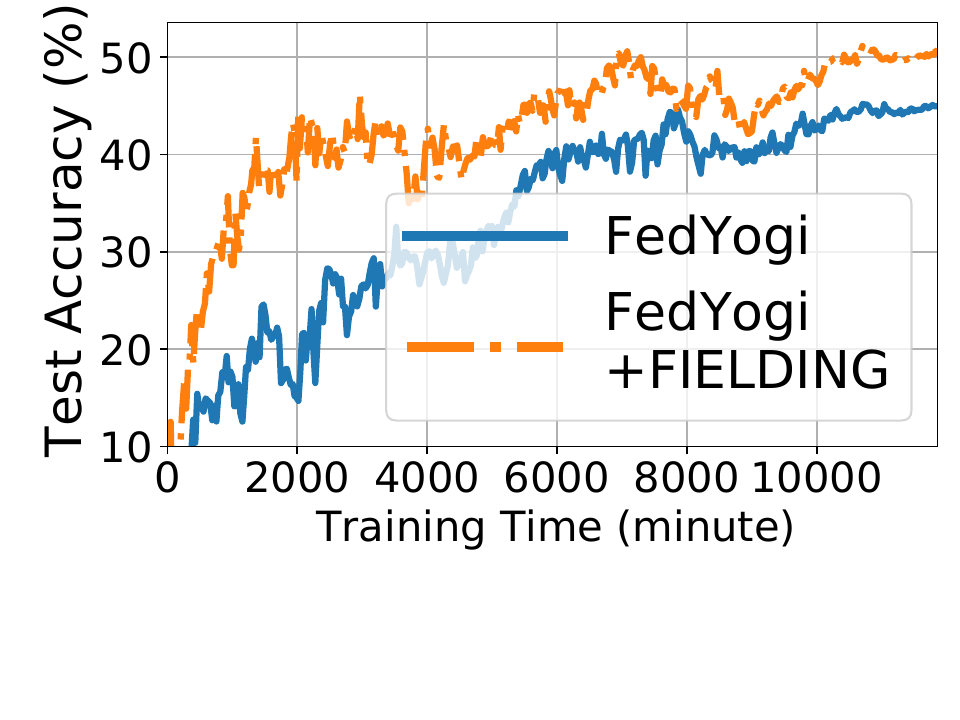}
            \label{fig:aggregation_fedyogi}
            \caption{FedYogi}
        \end{subfigure}
        \hfill
        \begin{subfigure}[t]{0.49\textwidth}
            \includegraphics[width=\linewidth, trim={0.2cm 2.3cm 0.38cm 0.1cm}, clip]{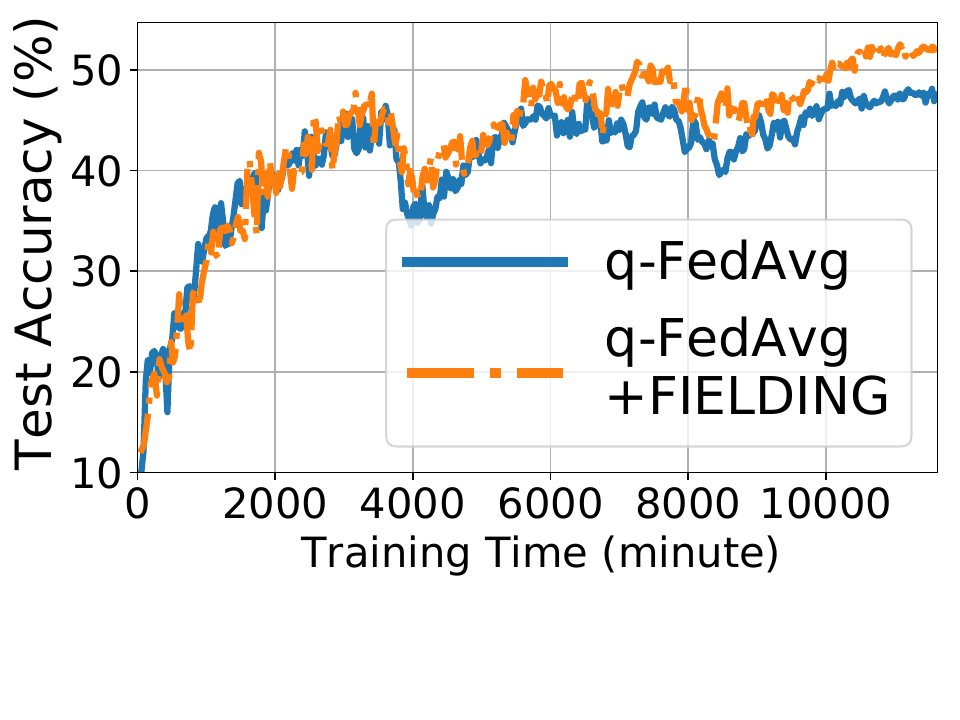}
            \label{fig:aggregation_qfedavg}
            \caption{q-FedAvg}
        \end{subfigure}
         \caption{\sysname with different FL algorithms on FMoW dataset.}
         \label{fig:aggregation}
    \end{figure}
    \end{minipage}
\end{figure}

\begin{figure}[t!]
    \begin{minipage}[t]{0.24\textwidth}
          \centering 
          \includegraphics[width=\textwidth, trim={0.4cm 2.3cm 0.38cm 0.1cm}, clip]{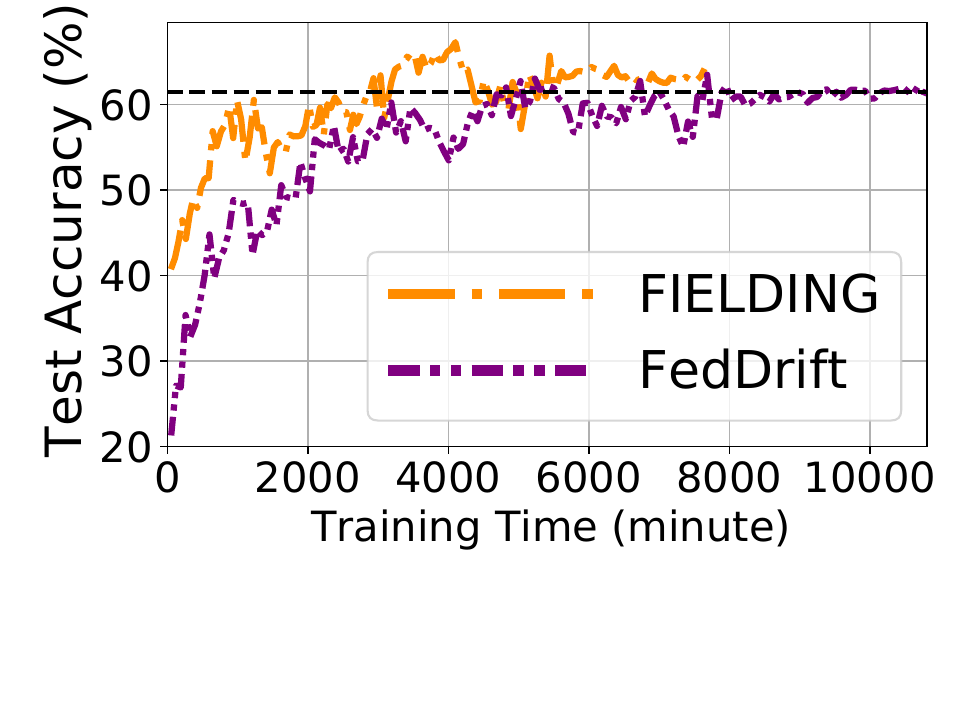} 
          \caption{Small-scale comparison against FedDrift on FMoW.}
          \label{fig:feddrift_comparison}        
    \end{minipage}
    \hfill
    \begin{minipage}[t]{0.25\textwidth}
      \centering 
      \includegraphics[width=\textwidth, trim={0.4cm 2.3cm 0.38cm 0.1cm}, clip]{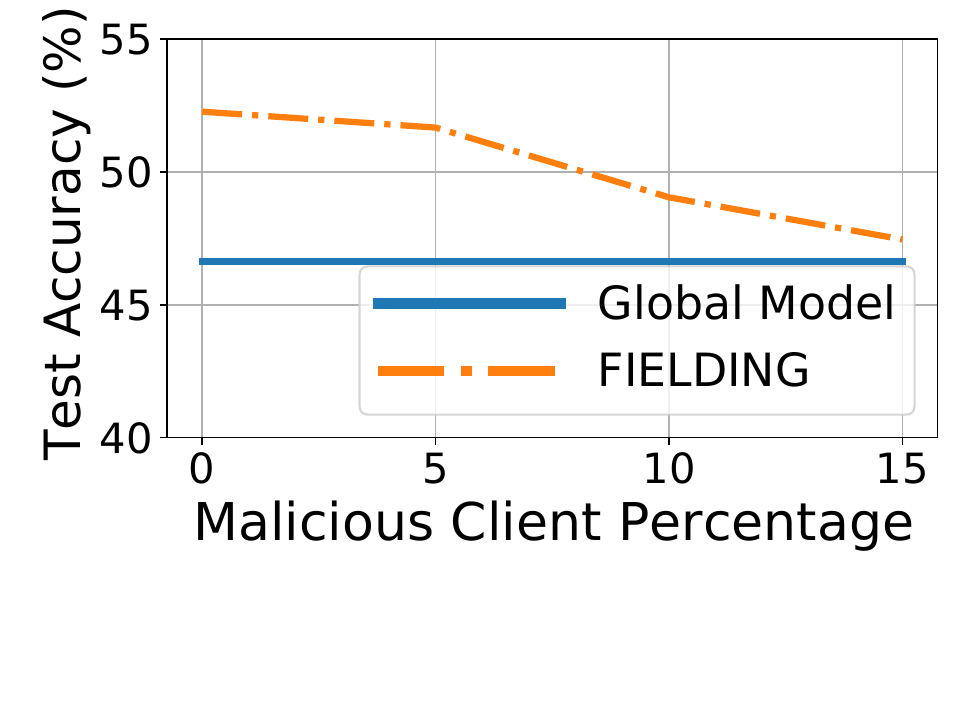} 
      \begin{minipage}[t]{1.2in}
          \caption{\sysname is robust to malicious clients.}
          \label{fig:robustness} 
      \end{minipage} 
    \end{minipage}
    \hfill
    \begin{minipage}[t]{0.25\textwidth}
      \centering 
      \includegraphics[width=\textwidth, trim={0.4cm 2.3cm 0.38cm 0.1cm}, clip]{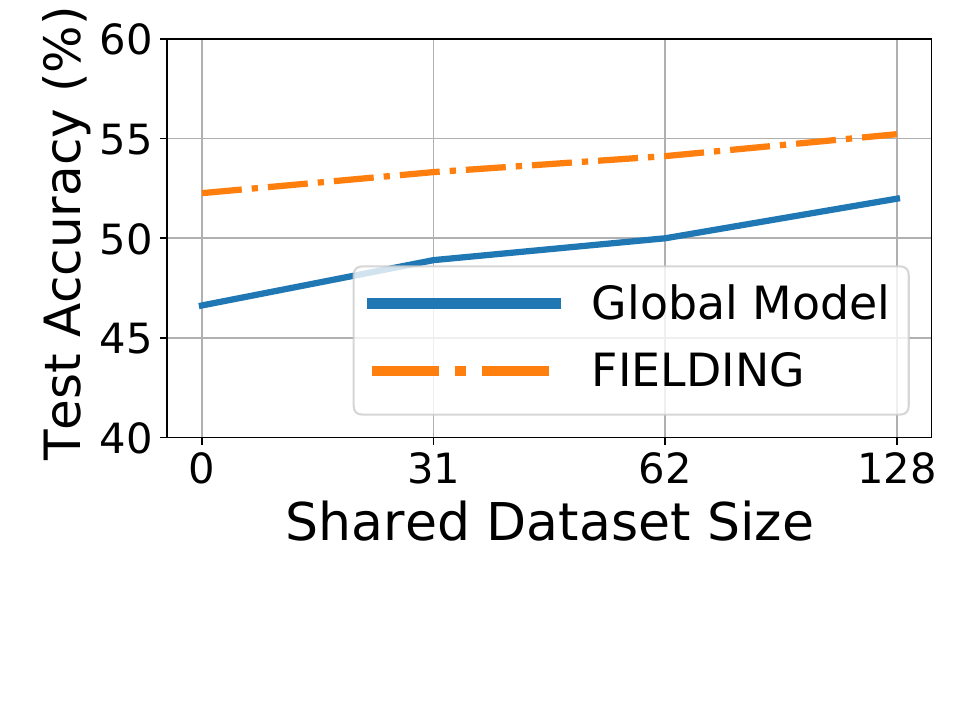}
      \begin{minipage}[t]{1.25in}
          \caption{\sysname offers gains across heterogeneity degrees.}
          \label{fig:sensitivity} 
      \end{minipage} 
    \end{minipage}
    \hfill
    \begin{minipage}[t]{0.24\textwidth}
      \centering 
      \includegraphics[width=\textwidth, trim={0.4cm 3.1cm 0.38cm 0cm}, clip]{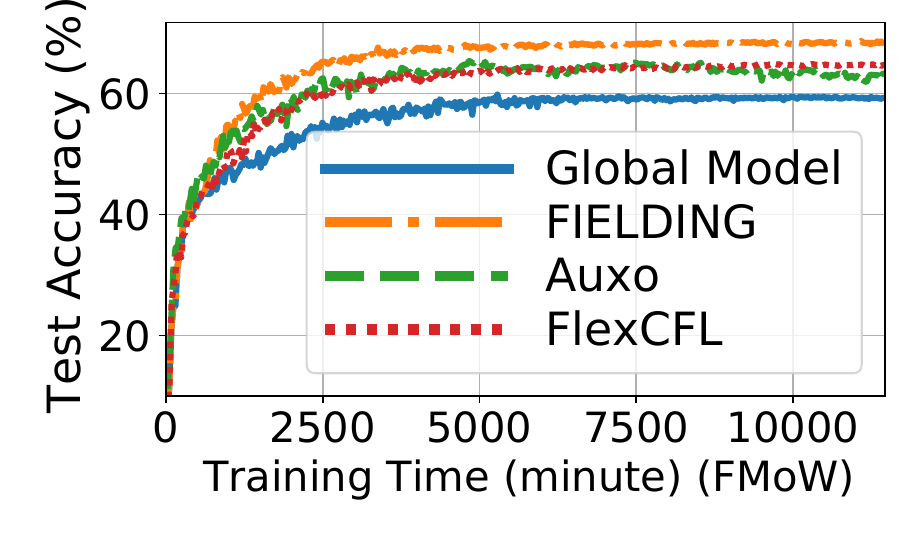} 
      \begin{minipage}[t]{1.2in}
          \caption{\sysname improves model accuracy on static data.}
          \label{fig:static_data} 
      \end{minipage} 
    \end{minipage}
\end{figure}

\textbf{Varying data heterogeneity.} Our analysis on \sysname's sensitivity to data heterogeneity degrees is inspired by the idea of improving training performance on non-IID data through a small shared dataset~\citep{noniddwithsharing}. We inject new training samples by creating three datasets shared with all clients: one sample for each of the least represented 50\% labels, one sample for each label, and two samples for each label. A larger shared set implies a smaller degree of heterogeneity. As shown in Figure ~\ref{fig:sensitivity}, both \sysname and the baseline global model benefit from this sharing approach, with \sysname improving the final average test accuracy by 3.2\% to 4.4\%.

\textbf{Static data.} Finally, we demonstrate that \sysname still improves model convergence when clients have static local data and no drifts happen. We rerun the experiment on the FMoW dataset with all data samples available throughout the training. As shown in Figure~\ref{fig:static_data}, \sysname offers the largest gain and improves the final average test accuracy by 9.2\%. Auxo and FlexCFL achieve an improvement of 3.9\% and 5.6\% respectively. Note that Auxo experiences an accuracy drop towards the end of training. This is likely due to Auxo re-clustering all selected clients after each round by default, causing unnecessary cluster membership changes when the data is completely static.

\section{Related works}
\label{sec:discussion}

\textbf{Heterogeneity-aware FL.} Recent works have also tackled heterogeneity challenges through client selection, workload scheduling, and update corrections. Oort considers both statistical and system utility and designs an exploration-exploitation strategy for client selection~\citep{oort}. PyramidFL adjusts the number of local iterations and parameter dropouts to optimize participants' data and system efficiency~\citep{pyramidfl}. DSS-Edge-FL dynamically determines the size of local data used for training and selects representative samples, optimizing resource utilization while considering data heterogeneity~\citep{dss_fl}. MOON corrects local updates on clients by maximizing the similarity between representations learned by local models and the global model~\citep{modelcontrasive}.

\textbf{Drifts-aware FL.} Adaptive-FedAvg handles concept drift by extending FedAvg with a learning rate scheduler that considers the variance of the aggregated model between two consecutive rounds~\citep{adaptivefedavg}. CDA-FedAvg detects concept drifts and extends FedAvg with a short-term and a long-term memory for each client. It applies rehearsal using data in the lone-term memory when drifts happen~\citep{continualfl}. Master-FL proposes a multi-scale algorithmic framework that trains clients across multiple time horizons with adaptive learning rates~\citep{onlinefl_dynamicregret}. They are complementary to our work and can improve cluster models when minor drifts that don’t trigger global re-clustering happen.

\sysname focuses on dynamic client clustering that is robust against potential data drift. It is agnostic to other FL optimizations and should work seamlessly with the FL aggregation, selection, and scheduling algorithms above.

\textbf{Clustered FL.} 
FlexCFL~\citep{flexcfl} and IFCA~\citep{ifca} adjust clusters by relocating shifted clients while maintaining a fixed number of clusters, operating under the assumption of low drift magnitude. Auxo~\citep{auxo} and FedAC~\citep{fedac} re-cluster selected clients using gradient information or local models, delaying drift adjustments for unselected clients. FedDrift~\citep{federateddistributeddrift} re-clusters all clients by broadcasting all cluster models every round to compute client local training loss, leading to significant costs. \Cref{sec:motivation} discusses other prior CFL works without drift handling and the impact of drifts on static clusters. 




\section{Conclusion}
\label{conclusion}

We presented \sysname, a clustering-based Federated Learning framework designed to handle multiple types of data drift with low overhead. Effective drift handling requires adapting to varying drift magnitudes while maintaining efficiency and supporting diverse drift types. To meet these goals, \sysname makes two key design choices: it combines per-client adjustment with selective global re-clustering and re-clusters all drifted clients using lightweight, drift-aware representations. \sysname dynamically adjusts the number of clusters and stabilizes the system’s response to minor drifts. It is robust to malicious clients and varying levels of heterogeneity, and remains compatible with a range of client selection and aggregation strategies. In \sysname, clients register with a centralized coordinator, which maintains metadata and assigns each client to the nearest cluster based on its local representation. Before each training round, the coordinator checks for shifts in cluster centers. If any center moves beyond a predefined threshold, global re-clustering is triggered; otherwise, the system proceeds with updated cluster assignments, significantly reducing the overall re-clustering cost and mitigating loss increase after re-clustering. Experimental results show that \sysname improves final test accuracy by 1.9\%–5.9\% over existing methods and provides greater stability under real-world data drift.

\textbf{Limitations.} There are potential directions to further enhance \sysname from the perspective of client representation.
Reporting these representations~-- label distributions, embeddings, or gradients~-- poses potential privacy concerns.
The differentially private approaches for private clustering are hard to apply in our settings since we perform clustering multiple times.
Additionally, gradient representation, while being able to capture concept drifts, leads to significant overhead.
We will address these concerns in future work.

\bibliographystyle{ACM-Reference-Format}


\newpage

\appendix

\sloppy
\section{Proof of convergence}
\label{app:convergence}

\begin{algorithm2e}[ht!]
    \caption{Clustered Federated Learning with Data Drift}
    \label{alg:clustered_sgd}
    \textbf{input:} the number of clusters $K$, the number of data drift events $T$, the number of sampled clients per round $M$, the number of rounds per data drift event $R$, reclustering threshold $\Delta$ \\
    Initialize cluster models $\C{1}{0}, \ldots, \C{K}{0}$ with the same model $\x{0}$ \\
    \For{$t = 0, \ldots, T-1$} {
        \For{each client $i$} {
            Adopt data drift for client $i$ \\
            Reassign client $i$ to the closest cluster \\
            Let $\x{i}$ be the model corresponding to the cluster containing client $i$
        }
        Let $\repr{1}{t}, \repr{2}{t}, \ldots$ be the client representations \\
        \eIf{there exists a cluster $C$ such that $\dist(\repr{i}{t}, \repr{j}{t}) > \Delta$ for clients $i,j \in C$} {
            Let $C_1, \ldots, C_K$ be the $K$-center clustering of all clients based on representations
        } {
            Let $C_1, \ldots, C_K$ be the current clusters
        }
        \For{$k=1,\ldots, K$} {
            $\tC{k}{0} \gets \avg_{i \in C_k} \x{i}$ \hfill\algcomment{Cluster model is the average of clients' models}
        }
        \For{$\tau = 0, \ldots, R - 1$} {
            \For{$k=1,\ldots, K$} {
                Sample subset $S$ from $C_k$ of size $M / K$ independently with replacement \\
                $\tC{k}{\tau + 1} \gets \tC{k}{\tau} - \step \avg_{i \in S} \G{i}{t}(\tC{k}{\tau})$
            }
        }
        \For{$k=1,\ldots, K$} {
            $\C{k}{t + 1} \gets \tC{k}{R}$
        }
    }
\end{algorithm2e}

In this section, we analyze the performance of our framework in the case when the number of clusters $K$ is known and the number of local iterations is $1$.
The analysis can be modified to handle local iterations and other variations based on the previous work, see e.g.~\citet{li_ConvergenceFedAvg_2020}.
To simplify the proof, in the algorithm, we make the following changes compared with the algorithm from the main body: 1) we use the $k$-center clustering instead of $k$-means clustering; 2) we perform reclustering when the intra-clsuter heterogeneity exceeds a certain threshold. We compare this threshold condition with the one from the main body in \Cref{app:reclustering_threshold}.

We present the algorithm in \Cref{alg:clustered_sgd}.
After every data drift event, we first accept data changes for each client.
After that, if the heterogeneity within some clsuter exceeds a certain threshold, we recluster the clients, and assign to each cluster a model by averaging models of all clients in the cluster.
After that, for several rounds, we perform standard federated averaging updates on each cluster: we sample $M$ clients, compute the stochastic gradient for each client, and update the cluster model with the sampled gradients.

\begin{assumption}[Objective functions]
    \label{ass:obj}
    Let $\f{i}{t}$ be the local function at each client $i$ at data drift event $t$.
    Then, for all $i$ and $t$:
    \begin{itemize}
        \item $\f{i}{t}$ is $L$-smooth: $\|\nabla \f{i}{t}(\vx) - \nabla \f{i}{t}(\vx)\| \le L \norm{\vx - \vy}$ for all $\vx, \vy$
        \item $\f{i}{t}$ satisfies $\mu$-Polyak-\L{}ojasiewicz ($\mu$-PL) condition: $\|\f{i}{t}(\vx)\|^2 \ge 2 \mu (\f{i}{t}(\vx) - \f{i}{t}(\vx^*))$, where $\vx^*$ is the minimizer of $\f{i}{t}$.
    \end{itemize}
\end{assumption}
The $\mu$-PL condition is a relaxation of the strong convexity assumption; both asumptions are standard in the federated learning literature~\citep{cho_ConvergenceFederated_2023, li_ConvergenceFedAvg_2020, cho_ClientSelection_2020}.
The following are standard assumptions on stochastic gradient.
\begin{assumption}[Stochastic gradients]
    \label{ass:sg}
    For each client $i$, at every data drift event $t$, let $\G{i}{t}$ be the stochastic gradient oracle:
    \begin{compactitem}
        \item $\Exp{\G{i}{t}(\vx)} = \nabla \f{i}{t}(\vx)$ for all $\vx$;
        \item $\ExpSqrNorm{\G{i}{t}(\vx) - \nabla \f{i}{t}(\vx)} \le \sigma^2$ for all $\vx$.
    \end{compactitem}
\end{assumption}
The next assumption connects the objective value with client representations.
Client representations can take various forms, such as label distributions, input embeddings, or gradients.
\begin{assumption}[Representation]
    \label{ass:representation}
    For a metric space $(X, \dist)$, let $\repr{i}{t} \in X$ be a representation of client $i$ at time $t$.
    Then there exists $\theta$ such that $\abs{\f{i}{t}(\vx) - \f{j}{t}(\vx)} \le \theta \cdot \dist\pars{\repr{i}{t}, \repr{j}{t}}$ for all $i,j,t$.
\end{assumption}
%
We next give the intuition of why the assumption is natural for the label-based and the embedding-based distnaces.
If the concept $P(\text{output} | \text{input})$ doesn't change between clients, every model achieves the same expected error on each class.
For each class, assuming that the class data for different clients is sampled from the same distribution, we get the same expected loss on this class for different clients.
Hence, the overall difference in the loss function is largely determined by the fraction of each class in the clients' data.
Similarly, if the client embedding faithfully captures information about its inputs, assuming the trained models depend on the inputs continuously, they have the same expected error on similar inputs.

Clearly, if data on clients can change arbitrarily at every data drift event, in general, there is no benefit in reusing previous iterates.
Hence, our next assumption bounds how much the client changes after the drift.
\begin{assumption}[Data drift]
    \label{ass:drift}
    At every data drift event, the representation of each client changes by at most $\delta$: $\dist\pars{\repr{i}{t}, \repr{i}{t+1}} \le \delta$ for all $i$, $t$.
\end{assumption}
This assumption naturally holds in case when, for each client, only a small fraction of data points changes at every round.
Finally, we assume that the clients are clusterable: there exist $K$ clusters so that within each cluster all points have similar representation.
\begin{assumption}[Clustering]
    \label{ass:clustering}
    At every data drift event, the clients can be partitioned into $K$ clusters $C_1, \ldots, C_K$ so that, for any $k \in [K]$ and any $i,j \in C_k$, we have $\dist\pars{\repr{i}{t}, \repr{j}{t}} \le \Delta$.
\end{assumption}
For completeness, we present the analysis of SGD convergence for functions satisfying $\mu$-PL condition.
\begin{lemma}
    \label{lem:standard_convergence}
    Let $f$ be $L$-smooth and $\mu$-strongly convex.
    Let $g$ be an unbiased stochastic gradient oracle of $f$ with variance $\sigma^2$.
    Let $\vx^*$ be the minimizer of $f$.
    Then, for the stochastic gradient update rule $\x{t+1} \gets \x{t} - \step g(\x{t})$ with $\step \le 1/L$, we have
    \begin{align*}
        \Exp{f(\x{T}) - f(\vx^*)}
        &\le \pars{1 - \step \mu}^T (f(\x{0}) - f(\vx^*)) + \frac{L \step}{2 \mu} \sigma^2.
    \end{align*}
\end{lemma}
\begin{proof}
    Let $\ExpArg{t}{\cdot}$ be expectation conditioned on $\x{t}$.
    By the Descent Lemma, we have:
    \begin{align*}
        \ExpArg{t}{f(\x{t+1})}
        &\le f(\x{t}) + \ExpArgInnerProd{t}{\nabla f(\x{t}), \x{t+1} - \x{t}} + \frac{L}{2} \ExpArgSqrNorm{t}{\x{t+1} - \x{t}}\\
        &= f(\x{t}) + \ExpArgInnerProd{t}{\nabla f(\x{t}), - \step g(\x{t})} + \frac{L}{2} \ExpArgSqrNorm{t}{\step g(\x{t})}\\
        &= f(\x{t}) - \step \SqrNorm{\nabla f(\x{t})} + \frac{L \step^2}{2} \pars{\SqrNorm{\nabla f(\x{t})} + \sigma^2} \\
        &\le f(\x{t}) - \frac{\step}{2} \SqrNorm{\nabla f(\x{t})} + \frac{L \step^2}{2} \sigma^2,
    \end{align*}
    where in the last inequality we used $\step \le 1/L$.
    Using the fact that $f$ satisfies $\mu$-PL condition, we have $\SqrNorm{\nabla f(\x{t})} \ge 2\mu(f(\x{t}) - f(\vx^*))$, giving
    \begin{align*}
        \ExpArg{t}{f(\x{t+1})}
        &\le f(\x{t}) - \step \mu (f(\x{t}) - f(\vx^*)) + \frac{L \step^2}{2} \sigma^2.
    \end{align*}
    Subtracting $f(\vx^*)$ from both parts, we have
    \begin{align*}
        \ExpArg{t}{f(\x{t+1}) - f(\vx^*)}
        &\le \pars{1 - \step \mu} (f(\x{t}) - f(\vx^*)) + \frac{L \step^2}{2} \sigma^2.
    \end{align*}
    By telescoping, taking the full expectation, and using $\sum_{i=0}^{\infty} \pars{1 - \step \mu}^i = \frac{1}{\step \mu}$, we get
    \begin{align*}
        \Exp{f(\x{T}) - f(\vx^*)}
        &\le \pars{1 - \step \mu}^T (f(\x{0}) - f(\vx^*)) + \frac{L \step}{2 \mu} \sigma^2.
        \qedhere
    \end{align*}
\end{proof}
The above result demosntrates how the objective improves with each round: namely, after $R$ iterations, the non-stochastic part of the loss improves by a factor which depends on $R$ exponentially.
However, each data drift and each reclustering might potentially hurt the objective.
Our next result bounds their effect on the loss.
\begin{lemma}
    \label{lem:with_clustering}
    Let $N$ be the number of clients, and $t+1$ be a fixed data drift event.
    Let $C_1, \ldots, C_k$ be clusters $\vc_1, \ldots, \vc_K$ be cluster models, and $\vc^*_1, \ldots, \vc^*_K$ be the optimal cluster models at the end of processing data drift event $t$.
    Similarly, let $\bar{C}_1, \ldots, \bar{C}_\ell$ be clusters $\bar{\vc}_1, \ldots, \bar{\vc}_K$ be cluster models, and $\bar{\vc}^*_1, \ldots, \bar{\vc}^*_K$ be the optimal cluster models immediately after reclustering.
    Then, under Assumptions~\ref{ass:obj}-\ref{ass:representation}, the following holds:
    \begin{align*}
        \frac{1}{N} \sum_{\ell \in [K]} \sum_{i \in \bar{C}_{\ell}} \pars{\f{i}{t+1}(\bar{\vc}_{\ell}) - \f{i}{t+1}(\bar{\vc}_{\ell}^*)}
        &\le \frac{1}{N} \sum_{k \in [K]} \sum_{i \in C_k} \pars{\f{i}{t}(\vc_{k}) - \f{i}{t}(\vc_k^*)}
        + 3 \theta (\Delta + \delta)
    \end{align*}
\end{lemma}

\begin{proof}
    For a client $i$, let $k$ and $\ell$ be such that $i \in C_k \cap \bar{C}_{\ell}$.
    We rewrite $\f{i}{t}(\bar{\vc}_\ell) - \f{i}{t}(\bar{\vc}_\ell^*)$ as
    \[
        \f{i}{t+1}(\bar{\vc}_\ell) - \f{i}{t+1}(\bar{\vc}_\ell^*)
        = \pars{\f{i}{t+1}(\bar{\vc}_\ell) - \f{i}{t+1}(\vc_k^*)} + \pars{\f{i}{t+1}(\vc_k^*) - \f{i}{t+1}(\bar{\vc}_\ell^*)}
    \]
    and bound each term separately.
    
    \paragraph{Bounding the first term}
    Let $\x{j}$ be the model of client $j$ before reclustering, i.e. $\x{j} = \vc_{k'}$ where $j \in C_{k'}$.
    Then,
    \begin{align}
        \f{i}{t+1}(\bar{\vc}_\ell)
        = \f{i}{t+1} \Bigl(\avg_{j \in \bar{C}_\ell} \x{j}\Bigr)
        \le \avg_{j \in \bar{C}_\ell} \f{i}{t+1}(\x{j}),
        \label{eq:split}
    \end{align}
    where the last inequality follows by convexity of $\f{i}{t}$.
    By \Cref{ass:clustering}, since all clients $\bar{C}_\ell$ belong to the same cluster, their representations differ by at most $\Delta$, and hence by Assumption~\ref{ass:representation} their local objectives differ by at most $\theta \Delta$.
    Therefore,
    \[
        \avg_{j \in \bar{C}_\ell} \f{i}{t+1}(\x{j})
        \le \avg_{j \in \bar{C}_\ell} (\f{j}{t+1}(\x{j}) + \theta \Delta)
        = \avg_{j \in \bar{C}_\ell} \f{j}{t+1}(\x{j}) + \theta \Delta
    \]
    Summing over all $i \in \bar{C}_\ell$, we get 
    \[
        \sum_{i \in \bar{C}_\ell}(\avg_{j \in \bar{C}_\ell} \f{i}{t+1}(\x{j}) + \theta \Delta)
        = \sum_{j \in \bar{C}_\ell} (\f{j}{t+1}(\x{j}) + \theta \Delta)
    \]
    By definition, for each $i \in C_k$ we have $\x{i} = \vc_k$.
    So,
    \[
        \sum_{\ell \in [K]} \sum_{i \in \bar{C}_\ell} \f{i}{t+1}(\bar{\vc}_\ell)
        \le \sum_{k \in [K]} \sum_{i \in C_k} \f{i}{t+1}(\vc_k)
    \]
    Hence, the sum of the first terms in \Cref{eq:split} over all $i$ can be bounded as
    \begin{align*}
        \sum_{k \in [K]} \sum_{i \in C_k} (\f{i}{t+1}(\vc_k) - \f{i}{t+1}(\vc_k^*) + \theta \Delta)
        &\le \sum_{k \in [K]} \sum_{i \in C_k} (\f{i}{t}(\vc_k) - \f{i}{t}(\vc_k^*) + \theta (\Delta + 2\delta)),
    \end{align*}
    where we used Assumptions~\ref{ass:representation} and~\ref{ass:drift} to bound the change of each objective after the data drift as $\theta \delta$.

    \paragraph{Bounding the second term}
    Let $\vx_{t}^{(i,*)}$ be the minimizer of $\f{i}{t}$ and $\vx_{t+1}^{(i,*)}$ be the minimizer of $\f{i}{t+1}$.
    Then, by Assumptions~\ref{ass:representation} and~\ref{ass:clustering} we have
    \begin{align*}
        \f{i}{t+1}(\bar{\vc}_{\ell}^*)
        &\ge \f{i}{t+1}(\vx_{t+1}^{(i,*)}) && \text{($\vx_{t+1}^{(i,*)}$ is the minimizer of $\f{i}{t+1}$)}\\
        &\ge \f{i}{t}(\vx_{t+1}^{(i,*)}) - \theta \delta && \text{(Assumptions~\ref{ass:representation} and~\ref{ass:drift})} \\
        &\ge \f{i}{t}(\vx_{t}^{(i,*)}) - \theta \delta && \text{($\vx_{t}^{(i,*)}$ is the minimizer of $\f{i}{t}$)}\\
        &\ge \f{i}{t}(\vc_{k}^*) - \theta (\Delta + \delta),
    \end{align*}
    \sloppy
    where the last inequality holds since otherwise $\f{i}{t}(\vc_{k}^*) > \f{i}{t}(\vx_{t}^{(i,*)}) + \theta \Delta$,
    which is impossible since $\vc_{k}^*$ is the minimizer of $\avg_{j \in C_k} \f{j}{t}$ and $\avg_{j \in C_k} \f{j}{t}(\vx_{t}^{(i,*)}) \le \f{i}{t}(\vx_{t}^{(i,*)}) + \theta \Delta$ by Assumptions~\ref{ass:representation}.

    Finally, we have $|\f{i}{t}(\vc_{k}^*) - \f{i}{t + 1}(\vc_{k}^*)| < \theta \delta$ by Assumptions~\ref{ass:representation} and~\ref{ass:drift}.
    Combining the bounds concludes the proof.
\end{proof}


%
Combining the above statements leads to our main convergence result.
\begin{theorem}
    Let $N$ be the number of clients, and let their objective functions satisfy Assumptions~\ref{ass:obj}-\ref{ass:representation}.
    Let $M$ be the total number of machines sampled per round.
    Let $\vx^*$ be the minimizer of $f_0 = \avg_i \f{i}{0}$.
    Let $\C{k,*}{t}$ be the minimizer for cluster $k$ at data drift event $t$.
    Then, for $\step \le 1/L$, for any data drift event $t$ of \Cref{alg:clustered_sgd} we have
    \begin{align*}
        \frac{1}{N} \sum_{k \in [K]} \sum_{i \in C_k} \pars{\f{i}{t}(\C{k}{t + 1}) - \f{i}{t}(\C{k,*}{t})}
        &\le (1 - \step \mu)^{TR} \pars{f(\x{0}) - f(\vx^*)} \\
        &\quad + \frac{L \step}{2\mu} \pars{\frac{\sigma^2 + 8 L \theta \Delta}{M/K}+ 3 \theta (\Delta + \delta) (1 - \step \mu)^{R}}
    \end{align*}
\end{theorem}

\begin{proof}
    First, we express the objective in terms of objectives for individual clusters:
    \begin{align*}
        \frac{1}{N} \sum_{k \in [K]} \sum_{i \in C_k} \pars{\f{i}{t}(\C{k}{t}) - \f{i}{t}(\C{k,*}{t})}
        &= \frac{1}{N} \sum_{k \in [K]} |C_k| \avg_{i \in C_k} \pars{\f{i}{t}(\C{k}{t}) - \f{i}{t}(\C{k,*}{t})}
    \end{align*}
    We then analyze the convergence in each cluster.
    \sloppy
    For any $L$-smooth function $h$ with minimizer $\vx^*$, it holds that $\|\nabla h(\vx)\| \le \sqrt{2 L (h(\vx) - h(\vx^*))}$.    
    First, note that, inside each cluster, the objective functions differ by at most $\theta \Delta$.
    Hence, for any two clients $i$ and $j$ from the same cluster, defining $h = \f{i}{t} - \f{j}{t}$, by $2L$-smoothness of $h$, for any $\vx$ we have
    \[
        \|\nabla h(\vx)\|
        \le \sqrt{4 L (h(\vx) - h(\vx^*))}
        \le \sqrt{8 L \theta \Delta}
    \]
    Hence, sampling clients uniformly from the cluster introduces variance at most $8 L \theta \Delta$.
    Since we sample $M/K$ clients from a cluster, the total variance~-- including the stochastic variance~-- is at most $\frac{\sigma^2 + 8 L \theta \Delta}{M/K}$.

    Next, by \Cref{lem:with_clustering}, the total objective increases by at most $3 \theta (\Delta + \delta)$ after reclustering.
    Since during $R$ rounds the non-stochastic part of each objective decreases by a factor of $(1 - \step \mu)^R$, we have
    \begin{align*}
        &\frac{1}{N} \sum_{k \in [K]} |C_k| \avg_{i \in C_k} \pars{\f{i}{t+1}(\C{k}{t+1}) - \f{i}{t+1}(\C{k,*}{t+1})} \\
        &\le \frac{(1 - \step \mu)^R}{N} \sum_{k \in [K]} |C_k| \avg_{i \in C_k} \pars{\f{i}{t}(\C{k}{t}) - \f{i}{t}(\C{k,*}{t}) + 3 \theta (\Delta + \delta)}
        + \sum_{i=0}^{R} \pars{1 - \step \mu}^i \cdot \frac{\sigma^2 + 8 L \theta \Delta}{M/K} \\
        &= (1 - \step \mu)^R 3 \theta (\Delta + \delta) + \sum_{i=0}^{R} \pars{1 - \step \mu}^i \cdot \frac{\sigma^2 + 8 L \theta \Delta}{M/K}  \\
        &\qquad + (1 - \step \mu)^R\frac{1}{N} \sum_{k \in [K]} |C_k| \avg_{i \in C_k} \pars{\f{i}{t}(\C{k}{t}) - \f{i}{t}(\C{k,*}{t})}
    \end{align*}
     By \Cref{lem:standard_convergence}, the last (stochastic) term accumulates over all data drift events as $\frac{L \step}{2 \mu} \cdot \frac{\sigma^2 + 8 L \theta \Delta}{M/K}$.
     By a similar reasoning, $3 \theta (\Delta + \delta)$ accumulates as $\frac{L \step}{2 \mu} \cdot 3 \theta (\Delta + \delta) (1 - \step \mu)^R$.
     And finally, $f(\x{0}) - f(\vx^*)$ term is multiplied by $(1 - \step \mu)^R$ at every data drift event, giving factor of $(1 - \step \mu)^{TR}$ over $T$ data drift events.
\end{proof}



\section{Additional background}
\label{sec:motivation}

    In this section, we provide the background of clustered federated learning and quantify the impacts of data drift. 

\subsection{Clustered federated learning}
\label{subsec:CFL}
FL deployment often involves hundreds to thousands of available clients participating in the training~\citep{google_fl}. For instance, \citet{fl_medicalrecords} have developed mortality and stay time predictors for the eICU collaborative research database~\citep{eicu} containing data of 208 hospitals and more than 200,000 patients. Although thousands of devices may be available in a given round, FL systems typically select only a subset for training. For example, Google sets 100 as the target number of clients per training round when improving keyboard search suggestions~\citep{googlekeyboard}. This selection process ensures a reasonable training time and accounts for diminishing returns, where including more clients does not accelerate convergence~\citep{fwdllm}.

Data heterogeneity is a major challenge in FL, as variations in local data volume and distribution among participants can lead to slow convergence and reduced model accuracy~\citep{characterizehetero, flsurvey}.
Clustering is an effective strategy to mitigate the impact of heterogeneity by grouping statistically similar clients. Clients within the same cluster collaboratively train one cluster model, which then serves their inference requests. Prior works demonstrate that clustering achieve high accuracy and fast model convergence~\citep{CFL_modelagnostic, personal_fl_approaches, ifca, auxo, federateddistributeddrift}.

Our experiments with the \sysname system prototype mainly use label distribution vectors as the client representation, as they handle both label and covariate shifts and are lightweight. Although previous studies have examined label distribution vectors as leverage to address heterogeneity~\citep{fedlabcluster, fedlc_zhang, fedlc_lee_seo, lee2021robust}, \sysname distinguishes itself from existing works. Federated learning via Logits
Calibration~\citep{fedlc_zhang} and FedLC~\citep{fedlc_lee_seo} are orthogonal and investigate building a better global model through client model aggregation and loss function improvements. FedLabCluster~\citep{fedlabcluster} and \citet{lee2021robust} do not explicitly discuss data drift handling.
Moreover, we present a theoretical framework with a convergence guarantee under data drift and re-clustering.


\subsection{Data drift affects static clustering} 


Drifts happen naturally when clients have access to non-stationary streaming data (e.g., cameras on vehicles and virtual keyboards on phones)~\citep{wilds, ekya, asynchronousfl, flmultimodalemotion, fldatastream}. For example, when training an FL model for keyboard suggestions, drifts occur when many clients simultaneously start searching for a recent event (widespread drift), a few clients pick up a niche hobby (concentrated intense drift), or a new term appears when a product is launched (new label added). 


Data drift reduces the effectiveness of clustered FL as it increases data heterogeneity within clusters. 
We demonstrate this problem with the Functional Map of the World (FMoW) dataset~\citep{fmow}. Recall from \Cref{sec:evaluation} that our preprocessed FmoW dataset contains time-stamped satellite images labeled taken on or after January 1, 2015, and each unique UTM zone metadata value corresponds to one client (302 clients in total). A day's worth of data becomes available every two training rounds, and clients maintain images they received over the last 100 rounds. 

As mentioned in \Cref{sec:design_decisions}, we use {\em mean client distance} as the intra-cluster heterogeneity. 
Note that instead of first averaging within each cluster and then finding the mean of those cluster averages, we use the overall mean across all clients. This helps us avoid biased results arising from imbalanced cluster sizes.
Note that in the baseline case without any clustering, we consider all clients as members of a "global cluster".

We cluster clients based on their label distribution vectors at the starting round and track the overall heterogeneity by measuring the mean client distance described above.
Figure~\ref{fig:fmow_intra_cluster_hetero} shows that initially (e.g., round 1), clustering reduces the per-cluster heterogeneity compared with no clustering (i.e., putting all clients in a global set). However, as data distribution shifts over time, static clustering increases per-cluster heterogeneity and soon gets close to the no clustering case at round 322.

\section{Implementation}
\label{implementation}

\begin{figure}[t]
    \centering 
    \includegraphics[width=0.7\linewidth, trim={0.53cm 1cm 0.27cm 0.1cm},clip]{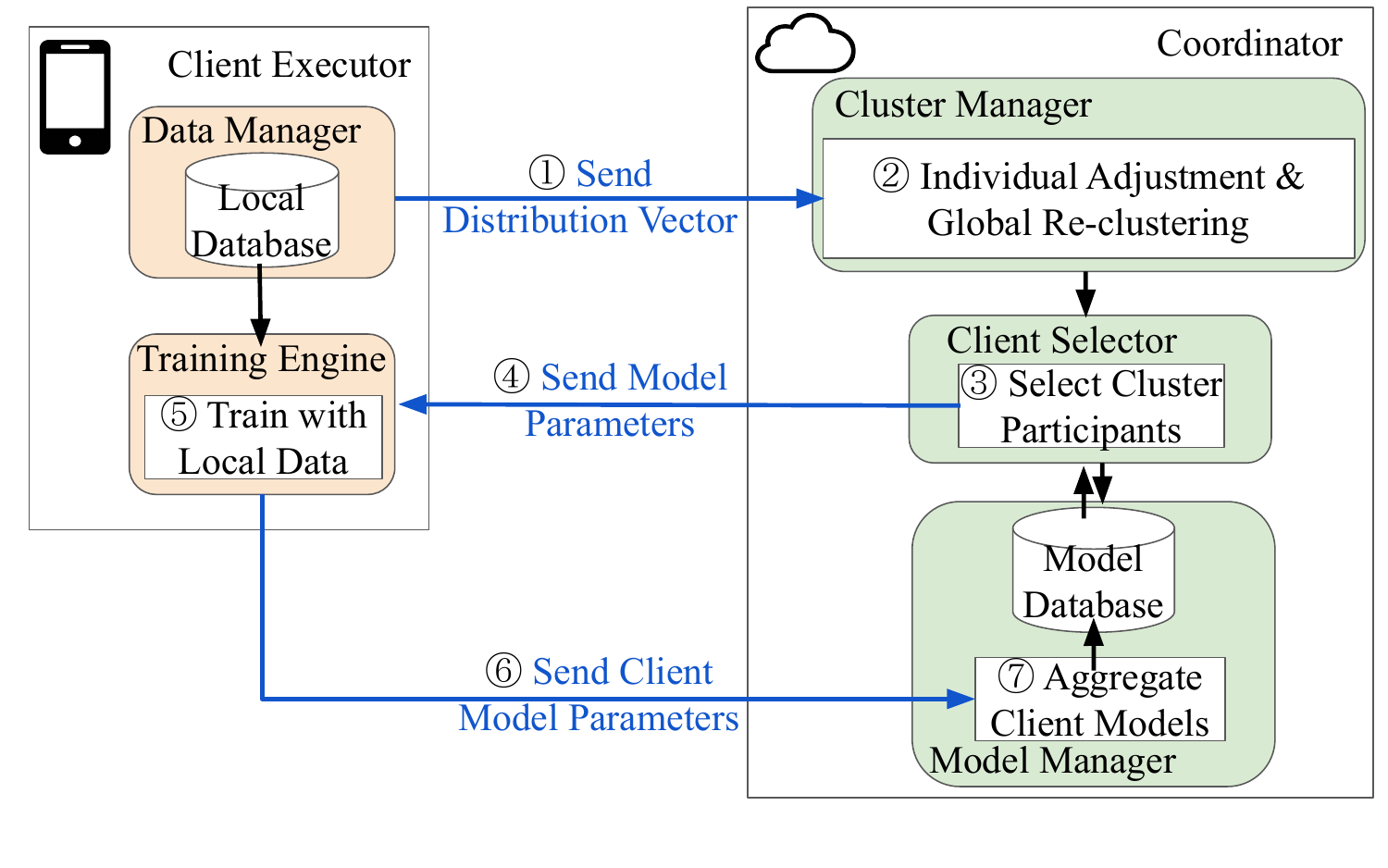}
    \caption{The system architecture and workflow of \sysname. Blue arrows represent communication between clients and the coordinator. Black arrows indicate the workflow among different components on the same machine.
    } 
    \label{fig:workflow}

\end{figure}

We implemented \sysname in Python on top of FedScale~\citep{fedscale}, a state-of-the-art open-source FL engine. We follow FedScale's design of having one centralized coordinator and client-specific executors. Figure~\ref{fig:workflow} illustrates the end-to-end workflow of the \sysname system prototype with label distribution vectors as client representations. At the beginning of each training round, \circled{1} client executors register with the coordinator to participate and report their local data distribution optionally. The coordinator's cluster manager maintains clients' distribution records and \circled{2} moves clients to the closest cluster when their reported distribution drifts. After moving drifted clients individually, the cluster manager measures center shift distances and either 
triggers global re-clustering when any cluster center shifts significantly or 
finalizes clusters membership. In the case of global clustering, we use $K$-means clustering and determine $K$ using the silhouette method~\citep{silhouettes}. The initial model of each newly created cluster is set as the average of its clients' previous cluster model (see \Cref{alg:data_drift}).

When client clustering is done, the coordinator notifies the client selector to \circled{3} select a subset of clients to contribute to each cluster and \circled{4} communicate the latest model parameters to the selected clients. \circled{5} Upon receiving a set of parameters, client executors conduct the training process over local data and \circled{6} upload the updated model parameters to the coordinator. Finally, \circled{7} the model manager aggregates individual models and stores the new cluster models. This process happens iteratively until the clients' mean test accuracy reaches a target value. \sysname also regularly creates checkpoints for the models, clients' metadata, and cluster memberships for future fine-tuning and failure recovery. \sysname's overheads are mainly determined by the number and size of distribution vectors we need to re-cluster. In our largest evaluation setting where we train with 5078 clients on a dataset with 100 labels, per-client adjustment takes 2.0 seconds and global re-clustering takes 15.6 seconds on average. Storing the latest reported distribution vector for each client consumes $5078\times100\times4$B $\approx 1.9$MB of memory.

\section{Experimental settings}
\label{app:eval_settings}
We conduct experiments on public datasets Functional Map of the World (FMoW)~\citep{fmow}, Cityscapes~\citep{cityscapes}, Waymo Open~\citep{waymo}, and Open Images~\citep{openimages}. In Table~\ref{table:dataset_setup}, we list the total number of clients we create on each dataset and the number of rounds in between clients getting new samples (e.g., as mentioned in Section~\ref{sec:evaluation}, on Cityscapes we partition each client's data into 10 intervals and stream in one interval every 30 rounds; so the rounds between new data arrivals are 30). In Table~\ref{table:training_parameter}, we specify the training parameters including the total number of training rounds, learning rate, batch size, number of local steps, and the total number of clients selected to contribute in each round.
\begin{table}[h]
\caption{Dataset configurations.}
\label{table:dataset_setup}
\vskip 0.15in
\begin{center}
\begin{small}
\begin{sc}
\begin{tabular}{P{0.2 \linewidth}P{0.25 \linewidth}P{0.3 \linewidth}}
\toprule
Dataset & Number of Clients & Rounds Between New Data Arrivals\\
\midrule
FMoW & 302 & 2\\
CityScapes  & 217 & 30\\
Waymo Open  & 212 & 20\\
Open Images & 5078 & 50\\
\bottomrule
\end{tabular}
\end{sc}
\end{small}
\end{center}
\vskip -0.1in
\end{table}

\begin{table}[h]
\caption{Training parameters.}
\label{table:training_parameter}
\vskip 0.15in
\begin{center}
\begin{small}
\begin{sc}
\begin{tabular}{P{0.18 \linewidth}P{0.08 \linewidth}P{0.1 \linewidth}P{0.05 \linewidth}P{0.18 \linewidth}P{0.16 \linewidth}}
\toprule
Dataset & Total Rounds & Learning Rate & Batch Size & Number of Local Steps & Participants per Round\\
\midrule
FMoW & 2000 & 0.05 & 20 & 20 & 50\\
CityScapes & 200 & 0.001 & 20 & 20 & 20\\
Waymo Open & 100 & 0.001 & 20 & 20 & 20\\
Open Images & 400 & 0.05 & 20 & 20 & 200\\
\bottomrule
\end{tabular}
\end{sc}
\end{small}
\end{center}
\vskip -0.1in
\end{table}

The URL, version information, and license of datasets we used are as follows:
\begin{itemize}
    \item FMoW: \url{s3://spacenet-dataset/Hosted-Datasets/fmow/fmow-rgb/}. fMoW-rgb version. This data is licensed under the \href{https://github.com/fMoW/dataset/raw/master/LICENSE}{Functional Map of the World Challenge Public License}.
    \item Cityscapes: \url{https://www.cityscapes-dataset.com/file-handling/?packageID=3} and \url{https://www.cityscapes-dataset.com/file-handling/?packageID=1}. Fine annotation version (5000 frames in total). The dataset is released under \href{https://www.cityscapes-dataset.com/license/}{Cityscapes' custom terms and conditions}.
    \item Waymo Open: \url{https://console.cloud.google.com/storage/browser/waymo_open_dataset_v_1_0_0}. Perception Dataset v1.0, August 2019: Initial release. The dataset is released under \href{https://waymo.com/open/terms/}{Waymo Dataset License Agreement for Non-Commercial Use}.
    \item Open Images: \url{https://fedscale.eecs.umich.edu/dataset/openImage.tar.gz}. Open Images (V7) Preprocessed by FedScale. The original Open Images dataset annotations are licensed by Google LLC under CC BY 4.0 license.
\end{itemize}

\section{Representations comparison}
\label{app:representation}

\subsection{Gradients as representation}

\begin{table*}[t]
\caption{Heterogeneity of all clients and intra-cluster heterogeneity after clustering (lower is better) on FMoW at various rounds. The left value denotes the pairwise L1 distance of distribution vectors, while the right represents the pairwise embedding squared Euclidean distance. Early on, gradient-based clustering is less effective due to the model being unstable, whereas label-based clustering decreases heterogeneity consistently. Notably, smaller distribution distances generally align with smaller embedding distances, suggesting that optimizing the label-based clustering objective also generates good embedding-based clusters.}
\label{table:cluster_gradient}
\begin{tabular}{|l|ll|ll|ll|}
\hline
Total Rounds & \multicolumn{2}{l|}{Un-Clustered} & \multicolumn{2}{l|}{Gradient-Based Clustering} & \multicolumn{2}{l|}{Label-Based Clustering} \\ \hline
100          & 1.81            & 46.33           & 1.82(+0.6\%)          & 46.30(-0.06\%)         & 1.65(-8.8\%)         & 43.02(-7.14\%)       \\
200          & 1.80            & 42.75           & 1.82(+1.1\%)          & 42.53(-0.51\%)         & 1.64(-8.9\%)         & 41.52(-2.88\%)       \\
500          & 1.78            & 36.99           & 1.71(-3.9\%)          & 35.82(-3.16\%)         & 1.62(-9.0\%)         & 35.55(-3.89\%)       \\
1000         & 1.74            & 31.85           & 1.66(-4.6\%)          & 30.70(-3.61\%)         & 1.56(-10.3\%)        & 29.97(-5.90\%)       \\
1500         & 1.76            & 32.35           & 1.63(-7.4\%)          & 30.68(-5.16\%)         & 1.60(-9.1\%)         & 30.76(-4.91\%)       \\ \hline
\end{tabular}
\end{table*}

An issue with gradient-based re-clustering is that the clustering effectiveness is sensitive to the model quality. Table~\ref{table:cluster_gradient} presents the resulting average pairwise distribution vector L1 distance and embedding
squared Euclidean distance when we perform gradient-based clustering after training a global model 
for various numbers of rounds on FMoW. The numbers in parentheses indicate the change in average distance relative to that of the global set. A negative value in the parentheses indicates that the generated clusters are overall \textit{less} heterogeneous than the global set. When we perform gradient-based clustering with the global model at round 100 and round 200, the resulting clusters are as heterogeneous as the global set. As we postpone clustering to later rounds, gradient-based clustering achieves an increasingly larger reduction in mean client distance. This trend indicates that gradient-based clustering doesn't perform well in earlier rounds when the global model we use to collect gradients from all clients has not converged.

As mentioned in \Cref{sec:design_decisions}, we propose using label distribution vectors as client representations when we deal with label and covariate shifts. Label distribution vectors are available on all clients and enable us to promptly detect any drift by checking for distribution changes. Compared to gradient-based clustering, label-based doesn't demand parameter or gradient transmission, doesn't introduce additional computation tasks, and is not sensitive to the timing of client clustering. In Table~\ref{table:cluster_gradient}, label-based clustering provides consistent heterogeneity reduction across rounds. Furthermore, as a label distribution vector typically has tens or hundreds of coordinates compared to millions in a full gradient, label-based clustering incurs significantly lower computational overhead. In our experiment of training ResNet-18 on the FMoW dataset, clustering all 302 clients using our label-based solution takes only 0.05 seconds while the gradient-based solution FlexCFL takes 37.92 seconds (note that FlexCFL already accelerates this process through dimensionality reduction using truncated Singular Value Decomposition). 

\begin{figure}[t!]
    \begin{minipage}[t]{0.45\textwidth}
      \centering 
      \includegraphics[width=\textwidth]{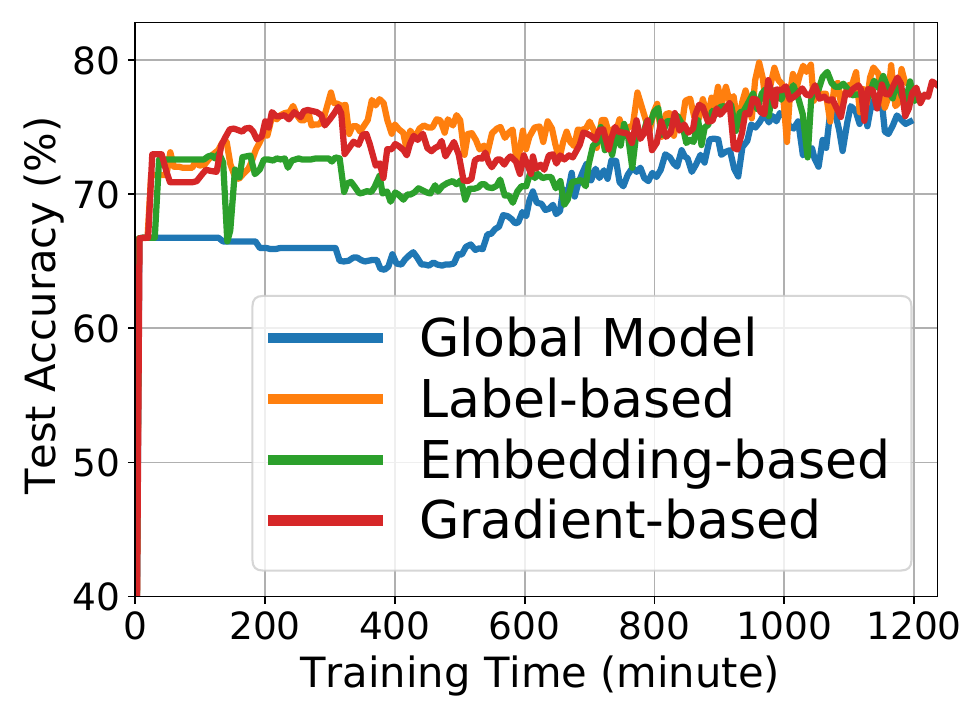} 
      \caption{\sysname achieves high accuracy across client representations on Cityscapes.}
      \label{fig:representation_compare} 
    \end{minipage}
    \hfill
    \begin{minipage}[t]{0.45\textwidth}
      \centering 
      \includegraphics[width=\textwidth]{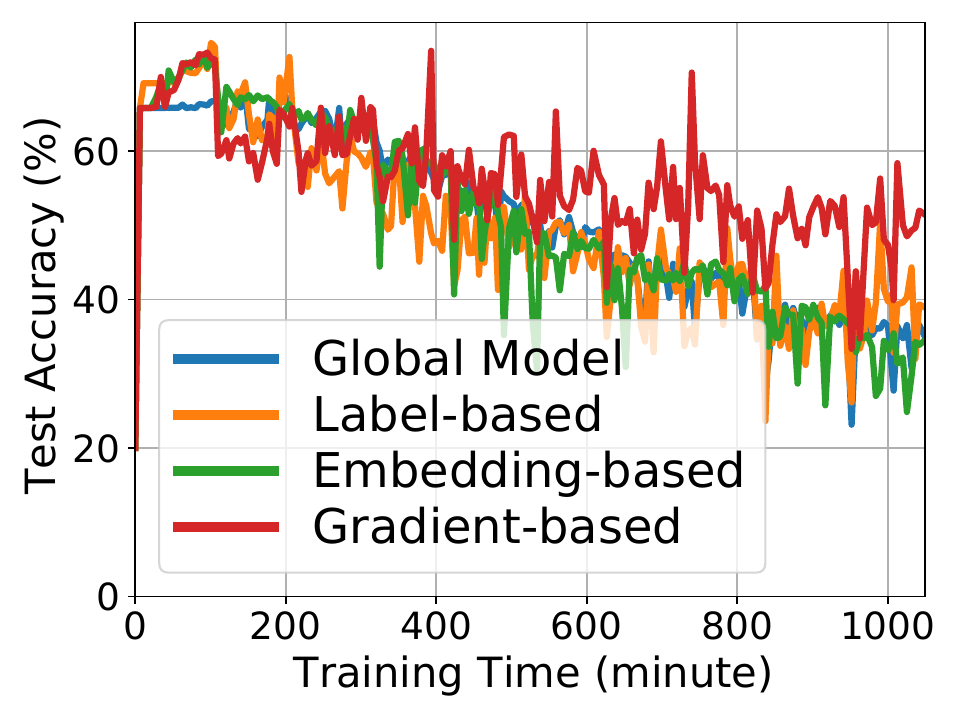}  
      \caption{Gradient-based clustering handles concept drift better.}
      \label{fig:concept_drift} 
    \end{minipage}
\end{figure}

\minghao{works better when the shared model used for gathering the embedding is not well-trained enough}
\minghao{higher overhead due to BP and larger volume of transmitted data (even assuming frozen model storing locally)}
\minghao{better for concept drift}

In terms of overhead, both embedding- and gradient-based clustering introduce client-side computation, which can be mitigated by deploying a smaller shared model for representation collection, eliminating extra model download time. We adopt this approach when we gather the time-to-accuracy results of training ResNet-18 on Cityscapes with various representations (shown in \Cref{fig:representation_compare}). We first construct a small training set by randomly sampling 200 images from each class. We then train a ResNet-18 model on this dataset for 300 epochs and broadcasts it so that clients store it locally and use it for gradient and embedding generation. Note that gradient-based \sysname still takes longer to finish 200 training rounds than other variants in \Cref{fig:representation_compare} due to the longer computational time of back propagation and longer transmission time of gradient vectors. 

Handling concept drifts where $P(y|x)$ changes requires loss-based representations, making gradients more appropriate than label distribution vectors or embeddings. \Cref{fig:concept_drift} shows the time-to-accuracy results with synthesized concept drifts on Cityscapes. Here we make all data samples available throughout the training, but introduce drift events by randomly choosing 50\% clients and having each chosen client randomly pick two labels and swap their samples (i.e., if a client picks label $A$ and $B$, then all samples previously labeled $A$ are now labeled $B$, and samples labeled $B$ now have label $A$). Gradient-based \sysname manages to retain test accuracy under such aggressive concept drift, while label and embedding-based \sysname have similar results as the baseline without any clustering. This result highlights that gradient has the potential of addressing concept drifts while label distribution and embedding don't. \minghao{concept drift is not our focus, might be able to improve the gradient based result further}

\subsection{Input embeddings as representations}
\minghao{comparable results. Embedding based slightly higher overhead due to FP}
Since covariate shift is defined as input distribution ($P(x)$) changing while the feature-to-label mapping ($P(y|x)$) remains constant, input embedding distance should be a proxy for label distribution distance. Hence, minimizing the distance between input embeddings naturally minimizes the label distribution distance as well. Table~\ref{table:cluster_gradient} supports this intuition by showing a correlation between the label distribution distance and the embedding distance. This observation suggests that using label distribution vectors or input embeddings as representations should have comparable performance on labeled datasets. As shown in \Cref{fig:representation_compare}, both label distribution vectors and embeddings enable \sysname to outperform the no-clustering baseline (the sudden accuracy drop of the embedding-based curve around 150 minutes is the result of global re-clustering). Input embeddings have the potential of capturing label and covariate shifts on unlabeled data, which label distribution vectors are incapable of.

\section{Additional results}
\label{app:other_results}
\begin{figure}[t]
      \centering 
      \includegraphics[width=0.6\textwidth, trim={0.38cm 0.4cm 0.38cm 0.38cm}, clip]{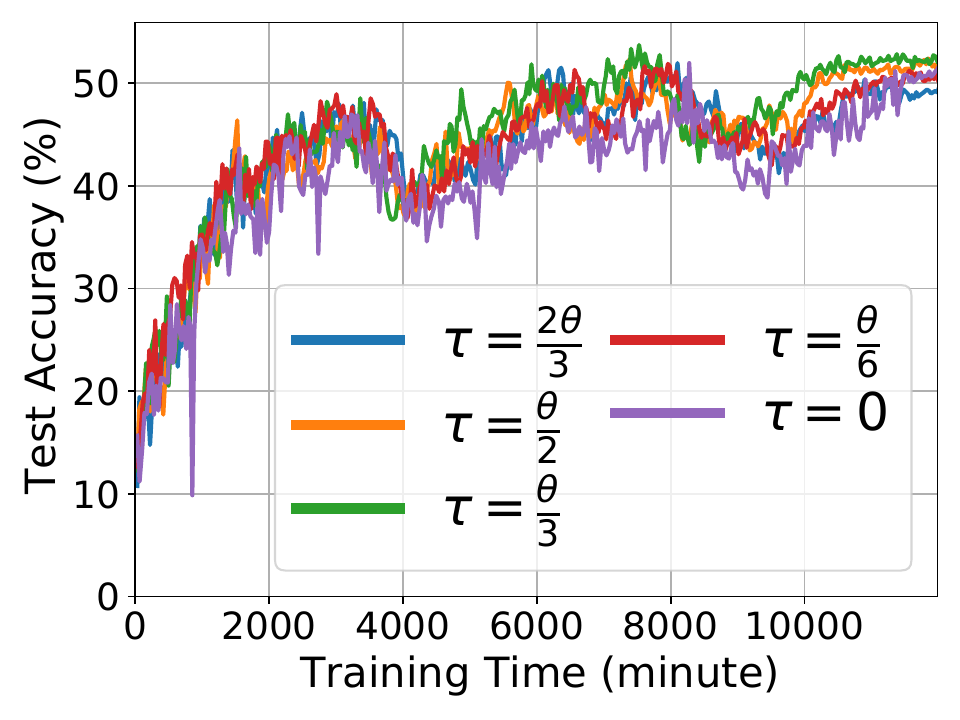}  
      \caption{\sysname performance with different global re-clustering threshold ($\tau$).}
      \label{fig:tau_ablation} 
\end{figure}
\subsection{Global clustering threshold ablation study}
A tunable parameter of our \sysname prototype is the global re-clustering threshold $\tau$. We present the results on FMoW with $\tau$ being 0 (i.e., global re-clustering by default), $\frac{\theta}{6}$, $\frac{\theta}{3}$, $\frac{\theta}{2}$, and $\frac{2\theta}{3}$ in Figure~\ref{fig:tau_ablation} (recall that $\theta$ is the average distance between cluster centers). $\tau$ equals 0 leads to sudden accuracy drops that match our observations in Figure~\ref{fig:global_vs_selective}. $\tau=\frac{\theta}{6}$ stabilizes the system, but the final accuracy is similar to that of $\tau=0$. The largest threshold $\tau=\frac{2\theta}{3}$ performs the worst, matching Figure~\ref{fig:fmow_intra_cluster_hetero} that migrating individual clients without adequate global re-clustering leads to more heterogeneous clusters. We choose $\tau=\frac{\theta}{3}$ for our prototype as it gives the highest final accuracy. However, we acknowledge that the optimal threshold might be workload-specific. To address this, we propose making $\tau$ learnable during training: in the early rounds, we explore various thresholds and select those yielding the best accuracy; we can also periodically re-learn the threshold in later rounds. This ensures good performance across different datasets.

\subsection{Alternative global clustering threshold using pairwise distance}
\label{app:reclustering_threshold}
\minghao{OK to place it here?}

To demonstrate that \sysname prototype performs empirically well with the alternate re-clustering thresholds presented in \Cref{app:convergence}, we re-run the Cityscapes task with re-clustering mechanism matching Line 9-11 of \Cref{alg:clustered_sgd}. In specific, we measure clients' pairwise distance as the L1 distance between their distribution vectors, and tune $\Delta$ adaptively. We start with $\Delta = c = 0.1$. After each data drift event, we check whether global re-clustering has been triggered consecutive by two drift events. If so, we update $\Delta$ to $2\times \Delta$; otherwise, we update $\Delta$ to $\min(c, \Delta-c)$. \Cref{fig:tau_global_thres} shows that the prototype of the \sysname system works well with both center shift distance-based and client pairwise distance-based global re-clustering triggering conditions. 


\begin{figure}[t!]
    \begin{minipage}[t]{0.45\textwidth}
      \centering 
      \includegraphics[width=\textwidth, trim={0.38cm 0.4cm 0.38cm 0.38cm}, clip]{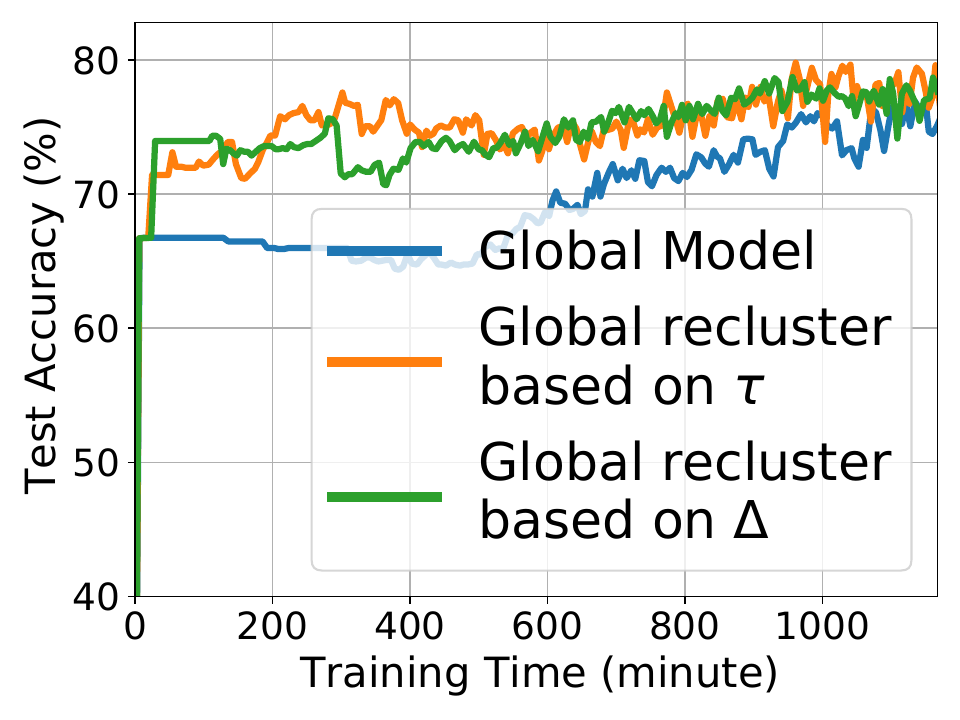}  
      \caption{\sysname performance with different global re-clustering triggering conditions (center shift distance as in \Cref{alg:data_drift} or client pairwise distance as in \Cref{app:convergence}).}
      \label{fig:tau_global_thres}
    \end{minipage}
    \hfill
    \begin{minipage}[t]{0.45\textwidth}
      \centering 
      \includegraphics[width=\textwidth, trim={0.38cm 0.4cm 0.38cm 0.38cm}, clip]{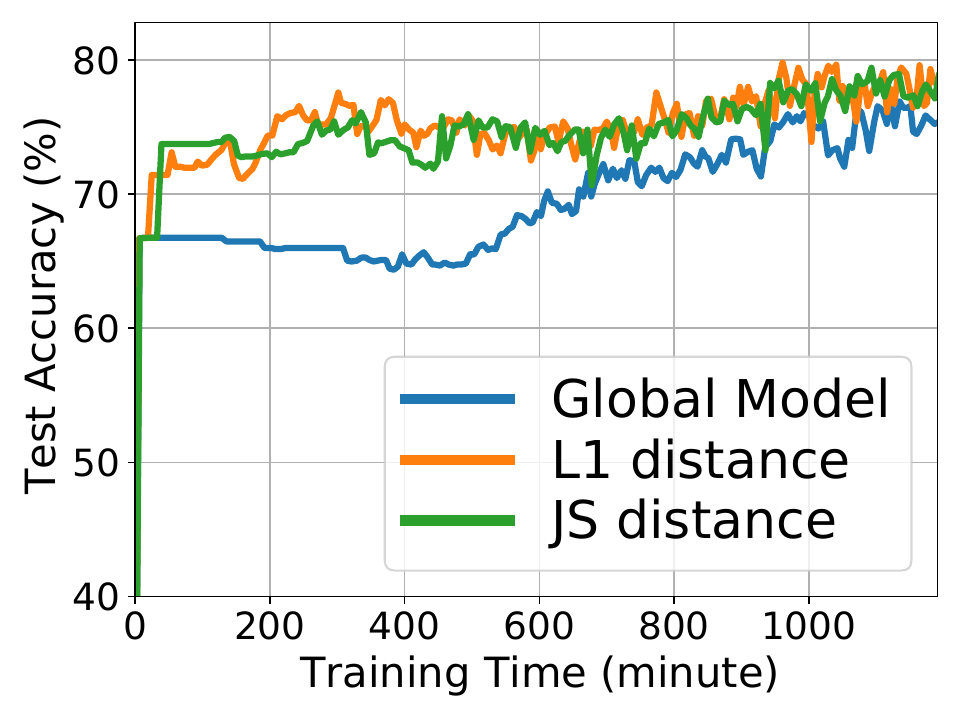}  
      \caption{\sysname performance with different distance metrics.}
      \label{fig:label_distance_metric} 
    \end{minipage}
\end{figure}

\subsection{Supporting diverse distance metric}
\minghao{add JS-distance results here}
To demonstrate that \sysname is not tied to specific distance metric, we run the Cityscapes task with label distribution vectors as client representations, and both L1 and Jensen–Shannon distance~\citep{jsdistance} as the distance metric. \Cref{fig:label_distance_metric} shows that \sysname is compatible with different distance metrics chosen for a given client representation.



\end{document}